\documentclass[11pt]{article}
\usepackage[utf8]{inputenc}

\usepackage[margin=1in]{geometry}

\usepackage{amsmath}
\usepackage{amsfonts}
\usepackage{amsthm}
\usepackage{xspace}
\usepackage{mathtools}

\usepackage{amssymb}
\usepackage{latexsym}   
\usepackage{graphicx}
\usepackage{color}
\usepackage[hidelinks]{hyperref}
\usepackage[linesnumbered,ruled,vlined]{algorithm2e}
\usepackage{booktabs}
\usepackage{times}
\usepackage{arydshln}
\usepackage{bm}
\usepackage{multirow}
\usepackage{makecell}
\usepackage{tikz}
\usetikzlibrary{decorations.pathreplacing,calligraphy}

\usepackage{blkarray}

\newcommand{\hide}[1]{}

\newtheorem{lemma}{Lemma}
\newtheorem{theorem}{Theorem}
\newtheorem{corollary}{Corollary}

\def\R{\mathbb{R}}

\def\inv{\mathrm{inv}}
\def\im{\mathrm{Im}}
\def\proj{\mathbf{P}}
\def\A{\mathcal{A}}
\renewcommand{\ker}{\mathrm{cker}}
\def\casvd{\mathsf{CA\textrm-SVD}\xspace}
\def\gkvsvd{\mathsf{GKV\textrm-SVD}\xspace}
\def\expm{\mathsf{EM}}
\def\M{\mathcal{M}}
\def\TV{\mathrm{TV}}
\def\Min{M^-}
\def\Mout{M^+}
\def\Mcut{Q}

\renewcommand{\dagger}[0]{+}

\newcommand{\spara}[1]{\paragraph{#1}}

\newcommand{\fabian}[1]{\textcolor{red}{[Fabian: #1]}}
\newcommand{\labis}[1]{\textcolor{red}{[Labis: #1]}}

\def\negfigsp{\vspace{0pt}}

\title{Learning Mixtures of Markov Chains with Quality Guarantees}
\date{October 2022}
\author{Fabian Spaeh\\ fspaeh@bu.edu \\ Boston University \and Charalampos E. Tsourakakis\\ ctsourak@bu.edu \\ Boston University}

\begin{document}

\maketitle

\begin{abstract}
A large number of modern applications ranging from listening songs online and browsing the Web to using a navigation app on a smartphone generate a plethora of user trails. Clustering such trails into groups with a common sequence pattern can reveal significant structure in human behavior that can lead to improving user experience through better recommendations, and even prevent suicides~\cite{lai2014caught}.  One approach to modeling this problem mathematically is as a mixture of Markov chains.  Recently, Gupta, Kumar and Vassilvitski~\cite{gupta2016mixtures} introduced an algorithm ($\gkvsvd$) based on the singular value decomposition (SVD) that under certain conditions can perfectly recover a mixture of $L$ chains on $n$ states, given only the distribution of trails of length 3 (3-trail).

In this work we contribute to the problem of unmixing Markov chains by highlighting and addressing two important constraints of the $\gkvsvd$ algorithm~\cite{gupta2016mixtures}: some chains in the mixture may not even be weakly connected, and secondly in practice one does not know beforehand the true number of chains. We resolve these issues in the Gupta et al. paper~\cite{gupta2016mixtures}. Specifically, we propose an algebraic criterion that enables us to choose a value of $L$ efficiently that avoids overfitting. Furthermore, we design a reconstruction algorithm that outputs the true mixture in the presence of disconnected chains and is robust to noise. We complement our theoretical results with experiments on both synthetic and real data, where we observe that our method outperforms the $\gkvsvd$ algorithm.  Finally, we empirically observe that combining an EM-algorithm with our method performs best in practice, both in terms of reconstruction error with respect to the distribution of 3-trails and the mixture  of Markov Chains. 
\end{abstract}

\section{Introduction}
\label{sec:intro} 
Since the seminal mathematical work of Markov~\cite{markov1906rasprostranenie} in the early 1900s, Markov chains~\cite{levin2017markov} have been central to computer science, physics, engineering, chemistry, and many more fields. For example, modeling user behavior on the Web primarily assumes Markovian behavior mainly due to Google's successful Pagerank algorithm~\cite{page1999pagerank}; Web pages are represented as states and hyperlinks as probabilities of navigating from one page to another. Despite evidence that this assumption is not fully accurate~\cite{chierichetti2012web}, Markov chains of higher order are used to represent the existence of users' memory~\cite{singer2014detecting}.   

Recently, Gupta, Kumar and Vassilvitski studied the problem of recovering a mixture of Markov Chains (MCs) from observations~\cite{gupta2016mixtures}.  Recovering a single Markov chain from observations is straightforward; the empirical starting distribution and the empirical transition probabilities form the maximum likelihood Markov chain. 
Despite a large body of related work on unmixing distributions~\cite{titterington1985statistical,lindsay1995mixture} such as Gaussian mixtures~\cite{chaudhuri2009learning,sanjeev2001learning,dasgupta1999learning} or mixture of DAGs~\cite{gordon2021identifying}, and the significant set of related applications~\cite{kohjima2021learning,gupta2016mixtures,singer2014detecting,becker2018understanding,wu2017retrospective,batu2004} the problem of unraveling mixtures of Markov chains was not studied until recently~\cite{gupta2016mixtures}.

Formally, we define a mixture $\M$ as the tuple $(M^1, M^2, \dots, M^L)$ of Markov chains, each given as a stochastic matrix of transition probabilities. Each chain $\ell \in [L]$ is associated with a vector of starting probabilities
$s^\ell \in \mathbb R^{n}$ such that $\sum_{\ell=1}^L \sum_{i=1}^n s^\ell_i = 1$. The goal is to learn  $\M$, i.e., the transition matrices and starting probabilities, from observations. Gupta et al.~\cite{gupta2016mixtures} prove that learning the mixture from trails of length 3 (3-trail) suffices for perfect reconstruction.  We sample a 3-trail $i \to j \to k$ from $\M$ as follows: We sample the starting state $i \in [n]$ and a chain $\ell \in [L]$ with probability $s^\ell_i$. From $i$, we travel through two more states $j$ and $k$, with transition probabilities from the chain $M^\ell$. In particular, the probability of obtaining a 3-trail $i \to j \to k$ also denoted as  the ordered tuple $(i, j, k)$ is $ p(i, j, k) \coloneqq    \sum_{\ell=1}^L s^\ell_i \cdot M^\ell_{ij} \cdot M^\ell_{jk}.$
 
Despite the fact that there exist solutions within the machine learning literature, Gupta et al. proved that under certain conditions (see Section~\ref{sec:rel}) one can perfectly recover the mixture, more time and space efficiently. In other words, their proposed algorithm $\gkvsvd$ compared to existing solutions, combines both strong theoretical guarantees and better efficiency. For example, one could use the EM algorithm~\cite{dempster1977maximum,wu1983convergence} to locally optimize the likelihood of the mixture, or learn a mixture of Dirichlet distributions~\cite{subakan2013probabilistic} without any strong theoretical guarantees on the quality of the output. Alternatively, one can use the moment based techniques that rely on  tensor and matrix decompositions to provably learn (under certain conditions) a mixture of Hidden Markov Models~\cite{anandkumar2012method,anandkumar2014tensor,subakan2013probabilistic} or Markov Chains~\cite[Section~4.3]{subakan2013probabilistic}. The method proposed by Gupta et al.~\cite{gupta2016mixtures} is significantly more efficient and scalable than the latter set of techniques which rely on 5-trails, and thus their sample complexity grows as $n^5$ rather than $n^3$ as the $\gkvsvd$ algorithm~\cite{gupta2016mixtures}. As the length of the trail grows, the recovery problem becomes easier. For example, if the length of a sample trail is large enough, then it suffices to learn the corresponding chain in the mixture.
As the length of the trail grows, the recovery problem becomes easier. For example, if the length of a sample trail is large enough, then we can immediately learn its underlying chain in the mixture. Kausik et al.~\cite{kausik2022learning} study a regime where the length of the samples is asymptotically at least as large as the worst mixing time of any chain in the mixture, and show that one can group trails by the chains they were sampled from and learn the mixture. 

In this work we focus on improving the seminal work of  Gupta et al.~\cite{gupta2016mixtures} in various ways. Gupta et al. require certain non-trivial conditions for reconstruction, stated as a requirement on the rank of a special matrix $\mathcal{A}$ that we discuss in greater detail in Section~\ref{sec:rel}. However, these conditions implicitly require all chains of the mixture to be connected.
Figure~\ref{fig:cex}(a) shows an instance of a mixture of two chains $M_1, M_2$ where $M_1$ is disconnected. The transition probabilities  are uniform  among all outgoing vertices and we use random starting probabilities. Both the $\gkvsvd$ algorithm due to Gupta et al.~\cite{gupta2016mixtures} and our reconstruction algorithm  use samples of 3-trails. We assume that we know the number of chains $L=2$. Our algorithm outputs the groundtruth mixture as shown in  Figure~\ref{fig:cex}(a), while $\gkvsvd$ outputs the mixture as shown in  Figure~\ref{fig:cex}(b). 
Recovering the mixture with $\gkvsvd$
fails; as we observe, the  $\gkvsvd$  ``blends''  the groundtruth chains into two connected 
chains. The visualization shows two types of mistakes in $\gkvsvd$: some transition probabilities in the output are lower than the groundtruth (gray), or non-zero when they should be zero (red). 
In practice, assuming connectivity is often too restrictive. For instance, a user may browse thematically coherent pages on a given topic and then randomly start surfing pages with a different theme.


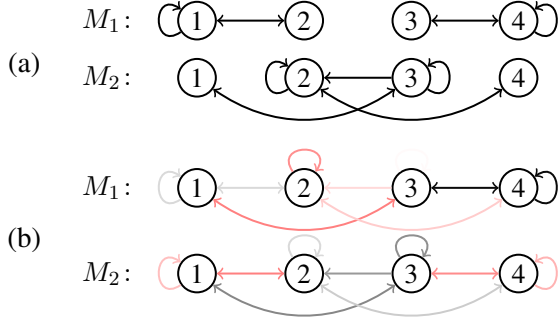
\begin{figure}[htbp]
\centering
\def\sndch{-0.8}%
\begin{tabular}{cm{7cm}}
(a) &
\scalebox{0.95}{\begin{tikzpicture}[state/.style={circle, draw=black, thick, minimum size=15pt, inner sep=0pt}]
    \node at (-1.2, 0) {$M_1\colon$};
    \node[state] (c1v1) at (0, 0) {1};
    \node[state] (c1v2) at (1.5, 0) {2};
    \node[state] (c1v3) at (3, 0) {3};
    \node[state] (c1v4) at (4.5, 0) {4};
    \draw[<->,thick] (c1v1) to (c1v2);
    \draw[<->,thick] (c1v3) to (c1v4);
    \draw[->,thick,looseness=4] (c1v4) to[out=-35,in=35] (c1v4);
    \draw[->,thick,looseness=4] (c1v1) to[out=215,in=145] (c1v1);
    
    \node at (-1.2, \sndch) {$M_2\colon$};
    \node[state] (c2v1) at (0, \sndch) {1};
    \node[state] (c2v2) at (1.5, \sndch) {2};
    \node[state] (c2v3) at (3, \sndch) {3};
    \node[state] (c2v4) at (4.5, \sndch) {4};
    \draw[<-,thick] (c2v2) to (c2v3);
    \draw[->,thick,looseness=4] (c2v3) to[out=-35,in=35] (c2v3);
    \draw[->,thick,looseness=4] (c2v2) to[out=215,in=145] (c2v2);
    \draw[<->,thick,looseness=1] (c2v1) to[out=-35,in=-145] (c2v3);
    \draw[<->,thick,looseness=1] (c2v2) to[out=-35,in=-145] (c2v4);
\end{tikzpicture}} \\[5mm]
(b) &
\def\sndch{-1.2}%
\scalebox{0.95}{\begin{tikzpicture}[state/.style={circle, draw=black, thick, minimum size=15pt, inner sep=0pt}]
    \node at (-1.2, 0) {$M_1\colon$};
    \node[state] (c1v1) at (0, 0) {1};
    \node[state] (c1v2) at (1.5, 0) {2};
    \node[state] (c1v3) at (3, 0) {3};
    \node[state] (c1v4) at (4.5, 0) {4};
    \draw[black!15,<->,thick] (c1v1) to (c1v2);
    \draw[black!96,<->,thick] (c1v3) to (c1v4);
    \draw[black!96,->,thick,looseness=4] (c1v4) to[out=-35,in=35] (c1v4);
    \draw[black!15,->,thick,looseness=4] (c1v1) to[out=215,in=145] (c1v1);
    \draw[red!44,->,thick,looseness=4] (c1v2) to[out=125,in=55] (c1v2);
    \draw[red!2,->,thick,looseness=4] (c1v3) to[out=125,in=55] (c1v3);
    \draw[red!50,<->,thick,looseness=1] (c1v1) to[out=-35,in=-145] (c1v3);
    \draw[red!20,<->,thick,looseness=1] (c1v2) to[out=-35,in=-145] (c1v4);
    \draw[red!15,<-,thick] (c1v2) to (c1v3);
    
    \node at (-1.2, \sndch) {$M_2\colon$};
    \node[state] (c1v1) at (0, \sndch) {1};
    \node[state] (c1v2) at (1.5, \sndch) {2};
    \node[state] (c1v3) at (3, \sndch) {3};
    \node[state] (c1v4) at (4.5, \sndch) {4};
    \draw[red!50,<->,thick] (c1v1) to (c1v2);
    \draw[red!45,<->,thick] (c1v3) to (c1v4);
    \draw[red!30,->,thick,looseness=4] (c1v4) to[out=-35,in=35] (c1v4);
    \draw[red!25,->,thick,looseness=4] (c1v1) to[out=215,in=145] (c1v1);
    \draw[black!15,->,thick,looseness=4] (c1v2) to[out=125,in=55] (c1v2);
    \draw[black!44,->,thick,looseness=4] (c1v3) to[out=125,in=55] (c1v3);
    \draw[black!47,<->,thick,looseness=1] (c1v1) to[out=-35,in=-145] (c1v3);
    \draw[black!20,<->,thick,looseness=1] (c1v2) to[out=-35,in=-145] (c1v4);
    \draw[black!40,<-,thick] (c1v2) to (c1v3);
\end{tikzpicture}} \\
\end{tabular}
\negfigsp
\caption{\label{fig:cex} Recovery of a Mixtures with Disconnections. 
(a) Ground-truth mixture. Our reconstruction algorithm perfectly reconstructs it from samples of 3-trails. 
(b) The output of $\gkvsvd$~\cite{gupta2016mixtures}. See the text for details.}
\end{figure}

\noindent A lesser issue of the $\gkvsvd$ algorithm is that it requires the choice of $L$. In practice, one can try the $\gkvsvd$ with different values of $L$ and find the smallest value that achieves near-zero reconstruction error (i.e., with respect to the input and learned 3-trail distributions), but one would then need to solve expensive matrix decompositions for each possible value of $L$. Can we find a good choice of $L$ more efficiently?   We prove bounds on the smallest non-zero singular value of certain matrices that can be used to guide the search for $L$ more efficiently in practice, avoiding the costly search.  
 
In summary, our key contributions include:

\begin{itemize}
    \item The method of Gupta et al.~\cite{gupta2016mixtures} cannot be used to learn mixtures that contain chains with disconnections, an important sub-class that occurs naturally. We extend their method to learn mixtures with disconnections.

    \item Oftentimes, little is known about the underlying mixture and it is thus unclear how to set $L$, the number of chains in the mixture. We show how to guess $L$ from the 3-trail distribution which is crucial for both SVD-based methods and iterative methods.  

    \item We verify our results experimentally and show that our method is well applicable under high levels of noise and on real-world data that is not exactly Markovian.
\end{itemize}

\spara{Outline} Section~\ref{sec:rel} introduces the notation and discusses the technical part of the GKV paper~\cite{gupta2016mixtures} that is central to our work. We also introduce some novel definitions in Section~\ref{sec:rel} that are crucial for our algorithmic improvements.  Then, Section~\ref{sec:proposed} presents our proposed techniques. The highlight is an improved reconstruction algorithm that   relaxes the implicit connectivity constraint of the GKV paper~\cite{gupta2016mixtures} and a well founded rule-of-thumb to select $L$, i.e., the total number of chains in the mixture. Section~\ref{sec:exp} presents our experimental results, both on synthetic and real data, and Section~\ref{sec:concl} concludes the paper with some directions for future work. For the convenience of the reader, we provide an example of the reconstruction in Appendix~\ref{sec:appendix2}.

\section{Theoretical preliminaries}
\label{sec:rel}

\begin{table}
  \caption{Symbols}
  \label{tab:sym}%
  \setlength{\tabcolsep}{5pt}%
  \centering
   \small{ 
  \begin{tabular}{rlp{10cm}}
    \toprule
    Matrix & &
    Definition \\
    \midrule
    $O_j$ & $n \times n$ & 3-trail distribution: $O_j(i, j) = p(i, j, k)$. \\
    $P_j$ and $Q_j$ & $L \times n$ & Probability of $i \to j$ and $j \to k$ in $M^\ell$:
    $P_j(\ell, i) = s^\ell_i \cdot M^\ell_{ij}$ and
    $Q_j(\ell, k) = s^\ell_j M^\ell_{jk}$.
    \\
    $S_j$ & $L \times L$ & Diagonal matrix with starting probabilities: $(S_j)_{\ell, \ell} = s^\ell_{j}$. \\
    \midrule
    $P'_j$ and $Q'_j$ & $L \times n$ & Guess at $P_j$ and $Q_j$ obtained by decomposing $O_j$. \\
    $X_j$ & $L \times L$ & Change of basis after decomposition. \\
    $Y_j$ and $Z_j$ & $L \times L$ & We defined $Y_j = X_j^{-1}$ and $Z_j = S_j \overline{X_j}$ such that $P_j = Y_j P'_j$
    and $Q_j = Z_j Q'_j$. \\
    \midrule
    $\A$ and $\A'$ & $2Ln\! \times\! n^2$ & Shuffle matrix $\A = \A(P_1, \dots, Q_1, \dots)$
    and $\A' = \A(P'_1, \dots, Q'_1, \dots)$. \\
    $D$ & $2Ln\! \times\! 2Ln$
    & Block diagonal matrix with diagonal blocks $Y_1, \dots, Y_n, Z_1, \dots, Z_n$ such
    that $\A = D \A'$. \\
    $Y'_j$ and $Z'_j$ & $r \times L$ & Rows of $(Y'_1, \dots, Z'_1, \dots)$ are a basis of $\ker(\A')$. \\
    \midrule
    $\xi_q$ & $2Ln$ & Indicator vector for the $q$-th conn. component $C \in \mathcal C^\ell$:
    $(\xi_q)^{\ell'}_x = 1 \Leftrightarrow x \in C \land \ell=\ell'$. \\
    $\Xi_j$ & $r \times L$ & Indicator matrix for state $j$:
    $\Xi_j(q, \ell') = 1$ $\Leftrightarrow j \in C \land \ell=\ell'$
    for the $q$-th connected component $C \in \mathcal C^\ell$. \\
    \midrule
    $R$ & $r \times r$ & Change of basis such that $\Xi_j Y_j = R Y'_j$ and $\Xi_j Z_j = R Z'_j$. \\
    $R_{\inv(j)}$ & $r \times L$ & $L$ columns of $R^{-1}$ corresponding to~conn. components
    containing $j$: $R_{\inv(j)}\! =\! R^{-1} \Xi_r$. \\
  \bottomrule
\end{tabular}
}
\end{table}

\spara{Notation}  We use some special notation throughout this paper
as we are dealing with numerous matrices: If needed, we write $A(i, j)$ to denote
the element $A_{ij}$ of a matrix $A$, in order to
avoid sub-indices. We also write $\overline{A}$
for the transpose of $A$ and $A^\dagger$ for the pseudoinverse.
Important matrices and their definitions are
conveniently found in Table~\ref{tab:sym}.

\subsection{GKV paper~\cite{gupta2016mixtures}}
\noindent \spara{The $\gkvsvd$ reconstruction algorithm} We now discuss in greater detail the method
of Gupta et al.~\cite{gupta2016mixtures} ($\gkvsvd$). A recovery algorithm is given access to the transition probabilities $p(i, j, k)$ for each 3-trail $(i, j, k) \in [n]^3$.
It will be more convenient to work with these probabilities in
matrix form, so we define $O_j \in \R^{n \times n}$ for
each intermediate state $j \in [n]$ as
\[
    O_j(i, k) \coloneqq p(i, j, k) .
\]
Following Gupta et al.~\cite{gupta2016mixtures},  we define the matrices $P_j, Q_j \in \R^{L \times n}$  for each state $j \in [n]$ as the probabilities
\[
    P_j(\ell, i) \coloneqq s^\ell_i \cdot M^\ell_{i, j}
    \quad\textrm{and}\quad
    Q_j(\ell, k) \coloneqq s^\ell_j \cdot M^\ell_{j, k}
\]
where $i, k \in [n]$ and $\ell \in [L]$. Note that the $i$-th column of $P_j$ and
the $j$-th column of $Q_i$  both contain the probability to start in state $i$  and transition to state $j$, in each chain of the  mixture, and are therefore identical.  As such, we call the sets $\{P_1, \dots, P_n\}$ and  $\{Q_1, \dots, Q_n\}$ \emph{shuffle pairs}.  Let also $S_j \in \R^{L \times L}$ be the diagonal matrix  containing the starting probabilities $(S_j)_{\ell \ell} \coloneqq s^\ell_j$ on its diagonal.
This enables Gupta et al.~\cite{gupta2016mixtures} to express $O_j$ compactly as 
\begin{align}
    \label{eq:6}
    O_j = \overline{P_j} \cdot S_j^{-1} \cdot Q_j.
\end{align}

\noindent  \spara{Singular Value Decomposition (SVD) as a recovery mechanism} Equation~\ref{eq:6} suggests that  we might be able to recover  $P_j$ and $Q_j$ by decomposing the
matrix $O_j$.
In particular, we
perform a singular value decomposition of $O_j$ into
$O_j = \overline{U_j} \cdot \Sigma_j V_j$
and define
\[
    P'_j \coloneqq U_j \in \R^{L \times n}
    \quad\textrm{and}\quad
    Q'_j \coloneqq \Sigma_j V_j .
\]
As shown by Gupta et al.~\cite{gupta2016mixtures},
$P'_j$ and $Q'_j$ are equal to $P_j$ and $Q_j$
up to a change of basis. That is, there exists a matrix
$X_j \in \R^{L \times L}$ of full rank such that 
$P_j = X_j^{-1} P'_j$ and $Q_j = S_j \overline{X_j} Q'_j$.
We also define
\[
    Y_j \coloneqq X_j^{-1}
    \quad\textrm{and}\quad
    Z_j \coloneqq S_j \overline{X_j}
\]
so the problem reduces to finding
$Y_j$ and $Z_j$.
We will further crucially use that
$\{P_1, \dots, P_n\}$ and $\{Q_1, \dots, Q_n\}$
are shuffle pairs, but it is not yet clear how
to use this combinatorial property.
In the next section, we will therefore show how
to express this property algebraically
through the co-kernel of a matrix $\A$.

\spara{The shuffle matrix $\A$}
We define the \emph{shuffle matrix}
as
$\A = \A(P_1, \dots, P_n, Q_1, \dots, Q_n)
\in \R^{2Ln \times n^2}$.
We index the rows of $\mathcal A$ with
triples from the set $[n] \times [L] \times \{ -, + \}$
where the symbols $-$ and $+$ stand for incoming and
outgoing edges, respectively.
We index columns with pairs from
the set $[n]^2$.
Using this convention, let
\begin{align*}
    \A_{(j,\ell,+), (i,j)} &\coloneqq P_j(\ell, i), \\
    \A_{(i,\ell,-),  (i,j)} &\coloneqq -Q_i(\ell, j),
\end{align*}
and all other entries be $0$.
The column $(i,j)$ of $\A$ therefore
contains the $i$-th column of $P_j$
and the $j$-th column of $-Q_i$.

We say
a vector $v \in \R^{2Ln}$ is in the
\emph{co-kernel} $\ker(\A)$
if $v^\top \A = \mathbf 0$.
We index the entries of $v$ like the
rows of $\A$, but write
for conciseness
$v^\ell_{j^+} \coloneqq v(j, \ell, +)$ and
$v_{j^+} \coloneqq (v_{j^+}^1, \dots, v_{j^+}^L)^\top \in \R^L$.
For the incoming counterparts,
we correspondingly write $v^\ell_{i^-}$ and $v_{i^-}$.
As promised,
we can now characterize the dependency
between matrices $P_j$ and $Q_i$
algebraically:
For a fixed chain $\ell \in [L]$,
we define a vector $v = w_\ell \in \R^{2Ln}$
and set $v^\ell_{j^+} = v^\ell_{i^-} \coloneqq 1$
for all $i, j \in [n]$
and all other entries to $0$.
Since
$(w_\ell^\top \A)_{(i,j)} = P_j(\ell, i) - Q_i(\ell, j)$,
we have that
the $i$-th column of $P_j$ is the
$j$-th column of $Q_i$ for all $i, j \in [n]$
if and only if
$w_\ell$ is in the co-kernel of $\A$, for all
$\ell \in [L]$.

\spara{The co-kernel of $\A'$}
To enforce the shuffle-pair property between the
matrices $P'_1, \dots, P'_n$ and $Q'_1, \dots, Q'_n$,
we naturally consider the co-kernel of the matrix
\[
    \A' \coloneqq \A(P'_1, \dots, P'_n, Q'_1, \dots, Q'_n).
\]
Recall that $P_j = Y_j P'_j$ and $Q_j = Z_j Q'_j$ for
all $j \in [n]$,
so we can directly relate $\A'$ to $\A$ via
the block diagonal matrix
\[\arraycolsep=0pt\def\arraystretch{.2}
    D \coloneqq \begin{pmatrix}
        Y_1 \vspace{-5pt} \\
        & \ddots \\
        & & Y_n \vspace{2pt} \\
        & & & Z_1 \vspace{-5pt} \\
        & & & & \ddots \\
        & & & & & Z_n \\
    \end{pmatrix}
    \in \R^{2Ln}
\]
through the equation $\A = D \A'$.
Note also that $D$ has full rank
since each of $Y_1, \dots, Y_n$
and $Z_1, \dots, Z_n$ has full rank.
It follows that
$v \in \ker(\A)$ if and only if
$\overline D v \in \ker(\A')$,
so we have full knowledge of the
co-kernel of $\A'$ given
the co-kernel of $\A$. Gupta et al.~\cite{gupta2016mixtures} refer to the mixture's property that suffices for perfect recovery as {\em well-distributed}, and it can be shown using standard linear algebra that no mixture can satisfy this condition unless $n\geq 2L$. We also inherit the latter condition.

\paragraph{Restrictions}

The original method by Gupta et al. \cite{gupta2016mixtures}
requires that each dimension of the co-kernel of $\A$ corresponds
to a single chain.
This restricts the mixture to be fully connected in each chain
as we show empirically in Figure~\ref{fig:cex}.
As such, their method is not applicable in many
natural scenarios where a chain might be disconnected.
For instance, a user might browse two sections of a website
separately as a simplified interface (e.g., on a smartphone)
might not allow a direct jump between both.
By relating the dimensions of the co-kernel of
$\A$ to connected components in the chains of the mixture,
we are able to alleviate exactly this connectivity
constraint.


\subsection{New Definitions and Comparison to $\gkvsvd$~\cite{gupta2016mixtures}}

To understand the connectivity of $M^\ell$ for
each $\ell \in [L]$,
we define the undirected bipartite graph $G^\ell$
that contains two copies $j^-$ and $j^+$ for every
state $j \in [n]$, and an edge from $i^+$
to $j^-$ if there is positive transition
probability from state $i$ to $j$, for
all $i, j \in [n]$.
More formally, let $V(G^\ell) = V^\pm$ be this
duplicated vertex set  and
\[
    E(G^\ell) \coloneqq \{ \{i^+, j^-\} \mid M^\ell_{ij} > 0 \} .
\]
We call a mixture \emph{companion-connected}
if each state $j \in [n]$ has a 
companion $i \in [n]$
such that $j^-$ and $i^-$ are reachable from $j^+$
in $G^\ell$ for all $\ell \in [L]$. We shall refer to this reachability property 
as $j^+$ being connected to $j^-$ and $i^-$. For example, consider the Markov chain (in the simple case $L=1$) defined by a random walk on a connected, undirected graph. Then, this chain is companion-connected, i.e., the reachability property holds, if and only if the graph is not bipartite; when the graph is bipartite, we will obtain two connected components in $G^1$ where the two copies of $j$ ($j^+$ and $j^-$) are in two different connected components.




We define $\mathcal C^\ell$ as the set of connected components
of $G^\ell$, for $\ell=1,\ldots,L$.
We need to connect the combinatorial
connectivity properties of the graph $G^\ell$ to the
algebraic properties of $\A$, which
we facilitate through
indicator vectors $v = \xi^\ell_C \in \R^{2Ln}$
for each 
$C \in \mathcal C^\ell$.
We set $v^\ell_x = 1$ if and only if $x \in C$ and
let entries corresponding to other chains be
all zero (i.e. $v^{\ell'}_x = 0$ for all
$\ell' \not= \ell$ and $x \in V^\pm$).

Let now $r \coloneqq \sum_\ell |\mathcal C^\ell|$ be the total
number of connected components in the mixture.
We identify each connected component
$C \in \mathcal C^\ell$, from each
chain $\ell \in [L]$, with an index $q \in [r]$. 
As such, we also write $\xi_q = \xi_C^\ell$ if
$C \in \mathcal C^\ell$ is the overall
$q$-th connected component in
the mixture.

Furthermore, if for each state $j$, the
vertices $j^+$ and $j^-$ are always
connected (e.g. if the mixture
is companion-connected), we may as well talk
about connectivity of states.
As such, we define matrices
$\Xi_j \in \R^{r \times L}$ for each state $j$
indicating whether vertices $j^+$ and $j^-$
belong to a connected
component $C \in \mathcal C^\ell$.
That is, for the $q$-th connected
component $C \in \mathcal C^\ell$, we let
$\Xi_j(q, \ell') = 1$ if $j \in C$
and $\ell = \ell'$, 
and $0$ otherwise.
Observe the duality
\begin{align}
    \label{eq:1}
    \begin{pmatrix}
        \xi^\top_1 \\
        \vdots \\
        \xi^\top_r
    \end{pmatrix} =
    \left(
        \Xi_1, \dots, \Xi_n, \Xi_1, \dots, \Xi_n
    \right)
\end{align}
as
$j^+$ and $j^-$ are always in the same connected
component and thus $\xi_{j^+} = \xi_{j^-}$
for all indicator vectors
$\xi \in \{\xi_1, \dots, \xi_r\}$. For the reader's convenience, we provide an example in Appendix~\ref{sec:appendix2} that illustrates the above concepts and definitions. 

\spara{What's new in our work}    Our method introduces non-trivial ideas and technical improvements that relax  the strict assumptions of \cite{gupta2016mixtures} about the co-kernel which force
the ground-truth mixture to be fully connected in every chain. This requires to prove further properties of the co-kernel of $\A'$ that allow us to identify which states are companions.  A further difference is that we have to carry out part of the reconstruction for each group of companions, and each group only reveals partial information. As such, we also show how to merge the individual reconstruction results and recover the mixture.

\section{Proposed method $\casvd$}
\label{sec:proposed}

Before stating our main theorem, we want
to motivate the connection between the co-kernel
of $\A$ and the connected components of the mixture.
With the previously introduced notation, we can write
\begin{align*}
    \| v^\top \A \|_2^2
    &= \sum_{i, j} \left( \sum_{\ell=1}^L
        \left( v_{j^+}^\ell P_j(\ell, i) - v_{i^-}^\ell Q_i(\ell, j) \right) \right)^2 = \sum_{i, j} \left( (v_{j^+} - v_{i^-})^\top S_i M_{ij} \right)^2
\end{align*}
where $(S_i M_{ij})_\ell = s^\ell_i \cdot M^\ell_{ij}$ is the
probability of starting in state $i$ of the $\ell$-th chain
and transitioning to state~$j$.
%
Given a vector $v \in \R^{2Ln}$ whose only
non-zero entries correspond to the $\ell$-th chain (i.e.
$v^{\ell'}_{j^-} = v^{\ell'}_{i^+} = 0$ for $\ell' \not= \ell$
and all $i, j \in [n]$),
the introduction of $G^\ell$ allows
us to write this norm
as the weighted Laplacian quadratic form
\[
    \| v^\top \A \|_2^2
    = \sum_{\{i^+, j^-\} \in E(G^\ell)}
        (s^\ell_i \cdot M^\ell_{ij})^2 \left(v^\ell_{j^-} - v^\ell_{i^+}\right)^2 .
\]
It follows from a basic fact of the Laplacian
form that there is one dimension
in the co-kernel of $\A$
for each connected component $C \subseteq V^\pm$.
In particular, this dimension is
spanned by the indicator vector $\xi^\ell_C \in \R^{2Ln}$.
It turns out that the relationship
between connected components and dimensions
in the co-kernel is crucial for the
recovery:

\begin{theorem}
    \label{thm:main}
    If the mixture is companion-connected,
    the co-kernel of $\mathcal A$ is
    spanned by indicator vectors
    $\xi_1, \dots, \xi_r$,
    and the ratios of starting probabilities
    $\{ s_i^\ell / s^\ell_j : \ell \in [L] \}$
    are distinct\footnote{This seems to be a necessary
    condition for $\gkvsvd$
    as well, so we are not more restrictive here.}
    for each pair $i, j \in [n]$,
    then we can reconstruct the mixture $\mathcal M$
    given its 3-trail distribution.
\end{theorem}

The conditions in Theorem~\ref{thm:main} are less restrictive than in
\cite{gupta2016mixtures}. Empirically all the real-world mixtures  we experimented with are  companion-connected, and lazy random walks on undirected graphs are provably 
companion-connected, as we explained before. 
%
%
We provide the pseudocode for our algorithm in Algorithm~\ref{alg:main} but its description is clear in the proof as it is constructive. Before delving into full details, we provide a proof outline of the reconstruction.

\begin{enumerate}
\item \emph{Co-Kernel}: As we know, $D$ maps between the         co-kernels of the two matrices $\A$ and $\A'$.  By assumption, vectors $\xi_1, \dots, \xi_r$ span the co-kernel of $\A$, so $\overline D \xi_1, \dots, \overline D \xi_r$ span the co-kernel of $\A'$. By Equation \eqref{eq:1} and definition of $D$, this basis can be written equivalently as the rows of
        the matrix $ \left(
                \Xi_1 Y_1, \dots, \Xi_n Y_n, \Xi_1 Z_1, \dots, \Xi_n Z_n
            \right)$.
%
We compute an arbitrary basis of the latter which is correct up to multiplication with a full-rank matrix $R \in \R^{r \times r}$. In particular,  $\Xi_j Y_j = R Y'_j$ for all states $j \in [n]$. 

\item \emph{Component Structure}:
        It turns out that we can efficiently
        determine whether two states $i$ and $j$
        are companions by looking at products of the form
        $(Z'_j \overline{Y'_j})^\dagger (Z'_i \overline{Y'_i})$
        and their pseudoinverses.
        We check companionship for all pairs
        of states $i, j \in [n]$, form equivalence
        classes of companionship, and arbitrarily pick a
        representative state $j$
        along with a companion $i$ from each class.

\item \emph{Eigendecomposition}:
        For each representative $j$ and companion $i$, we
        have $\Xi_i = \Xi_j$ and thus obtain  
        \[
            (Z'_j \overline{Y'_j})^\dagger (Z'_i \overline{Y'_i})
            = \overline{R^\dagger_{\inv(j)}} S_j^{-1} S_i \overline{R_{\inv(i)}}
        \]
        where $R_{\inv(j)} \coloneqq R^{-1} \Xi_j$ are the 
        columns of $R^{-1}$ that correspond to connected
        components containing $j$.
        We perform an eigendecomposition of the
        above matrix which reveals $R_{\inv(j)}$ 
        up to a scaling and ordering of the eigenvectors.
        We use a fact from Gupta et al. \cite{gupta2016mixtures}
        to undo the scaling which gives
        us a matrix $\tilde R_j = \Pi^{-1}_j R^\dagger_{\inv(j)}$
        for some arbitrary rotation matrix $\Pi_j$.
        
\item   \emph{Merging Components}:
        Each equivalence class reveals some
        columns of $R^{-1}$, so overall
        we know all columns of $R$, just not their
        order. We combine the columns and take the inverse
        which yields $\Pi R$ for another
        permutation matrix $\Pi$.
        Note that $\Xi_j Y_j$
        is just $Y_j$ but interspersed with
        all-zero rows.
        With some care we can undo the effect
        of $\Xi_j$ uniformly for all states,
        and then easily reconstruct the whole
        mixture from there.

\end{enumerate}

%
\begin{algorithm}
\label{alg:main}
\newcommand*\ruleline[1]{\noindent\raisebox{.8ex}{\makebox[\linewidth]{\hrulefill\hspace{1ex}\raisebox{-.8ex}{#1}}}}
\newcommand\mycommfont[1]{\it\small#1}
\SetCommentSty{mycommfont}
\SetKwComment{Comment}{}{}
\DontPrintSemicolon
\caption{$\casvd$}\label{alg:two}
\KwData{3-trail distribution $p$}
\KwResult{Mixture $\mathcal M$}
\For(\Comment*[h]{\hfill SVD}){$j \in [n]$}{
    Let $O_j(i, k) = p(i, j, k)$ for all $i, k \in [n]$\;
    Decompose $O_j$ into $\overline{U_j} \Sigma_j V_j$ via SVD\;
    $P'_j \gets U_j$\;
    $Q'_j \gets \Sigma_j V_j$\;
}
$\A' \gets \A(P'_1, \dots, P'_n, Q'_1, \dots, Q'_n)$ \Comment*{Co-Kernel}
Let $(Y'_1, \dots, Y'_n, Z'_1, \dots, Z'_n)$ be any basis for $\ker(\A')$\;
$S \gets [n], J \gets \emptyset$ \Comment*{Component Structure}
\While{$S \supsetneq \emptyset$}{
    Remove an arbitrary element $j$ from $S$ and add $j$ to $J$\;
    \For{$i \in S$}{
        \If{$\big( (Z'_i \overline{Y'_i})^\dagger (Z'_j \overline{Y'_j}) \big)^\dagger = (Z'_j \overline{Y'_j})^\dagger (Z'_i \overline{Y'_i})$}{
            Mark $i$ as a companion of $j$ and remove $i$ from $S$
        }
    }
}
\For(\Comment*[h]{\hfill Eigendecomposition}){$j \in J$}{
    Let $i$ be any companion of $j$\;
    Decompose $(Z'_j \overline{Y'_j})^\dagger (Z'_i \overline{Y'_i})$ into $(\tilde R'_j)^{-1} \Lambda \tilde R'_j$ via an
    eigendecomposition\;
    $d_j \gets \overline{O_j \mathbf 1_n} \cdot \big( \tilde R'_j Y'_j P'_j \big)^\dagger \in \R^L$\;
    $\tilde R_j \gets \mathrm{diag}(d_j) \cdot \tilde R'_j$\;
}
Let $\rho_1, \dots, \rho_r$ be distinct columns from $\{ \tilde R_j : j \in J \}$\;
$R \gets \left(\rho_1\ \rho_2\ \dots\ \rho_r\right)^{-1}$ \Comment*{Merging Components}
\For{$j \in J$}{
    Let $C(j) \subseteq [r]$ be the indices of non-zero rows in $R Y'_k$\;
}
Find an assignment $a\colon [r] \to [L]$ such that $\{ a(q) : q \in C(j) \} = [L]$ for all $j \in J$\;
\For{$j \in [n]$}{
    Initialize empty matrices $Y_j, Z_j \in \R^{L \times L}$ with
    rows $Y_j(1), \dots, Y_j(L)$ and $Z_j(1), \dots, Y_j(L)$\;
    \For{$q \in [r]$}{
        \If{$(R Y'_k)(q) \not= 0$}{
            $Y_j(a(q)) \gets (R Y'_k)(q)$\;
            $Z_j(a(q)) \gets (R Z'_k)(q)$\;
        }
    }
    $S_j \gets Z_j Y_j$\;
    $P_j \gets Y_j P'_j$\;
    \For{$\ell \in [L], i \in [n]$}{
        $M^\ell_{ij} \gets P_j(\ell, i) / S_j(\ell, \ell)$
    }
}
\end{algorithm}

\subsection{Co-Kernel}


By assumption, the co-kernel of $\A$ is spanned by
indicator vectors $\xi_1, \dots, \xi_r$
which implies that the co-kernel of
$\A'$ is spanned by
$\overline D \xi_1, \dots \overline D \xi_r$.
Using Equation \eqref{eq:1}, we can
succinctly represent this basis
of the co-kernel of $\A'$
as the rows of the matrix
\begin{align}
    \label{eq:2}
    \left(
        \Xi_1, \dots, \Xi_n, \Xi_1, \dots, \Xi_n
    \right) \cdot D
    =
    \left(
        \Xi_1 Y_1, \dots, \Xi_n Y_n, \Xi_1 Z_1, \dots, \Xi_n Z_n
    \right)
    \in \R^{r \times 2nL} .
\end{align}
We take a guess at this by computing
arbitrary basis of the co-kernel of $\A'$,
given as the rows of the matrix
\begin{align}
    \label{eq:3}
    \left(
        Y'_1, \dots, Y'_n, Z'_1, \dots, Z'_n
    \right)
    \in \R^{r \times 2nL}.
\end{align}
Since the rows of matrices in
(\ref{eq:2}) and (\ref{eq:3})
are both bases for the co-kernel,
there necessarily
exists a matrix
$R \in \R^{r \times r}$ of full rank such that
$\Xi_j Y_j = R Y'_j$ and
$\Xi_j Z_j = R Z'_j$.
By using the shuffle pair property in the
form of the co-kernel of $\A$ and relating
this co-kernel to $\A'$, we have
reduced the problem to finding $R$.
In the next section, we show how to
explore the component structure of the
mixture to recover parts of $R$.

\subsection{Component Structure}

For each $j \in [n]$, let us define $ R_{\inv(j)} \coloneqq R^{-1} \Xi_j \in \R^{r \times L}$ 
as the inverse of $R$, restricted to columns
corresponding to connected components containing
$j$.
Note that $R_{\inv(j)}$ has full column rank since
$\Xi_j$ also has full column rank.
Matrices $Z_j$ and $\overline{Y_j}$ are inverse
up to $S_j$, so we can also try and
multiply
\[
    Z'_j \cdot \overline{Y'_j}
    = R^{-1} \Xi_j Z_j \cdot
    \overline{Y_j} \, \overline{R^{-1} \Xi_j}
    = R_{\inv(j)} S_j \overline{R_{\inv(j)}}.
\]
Since $R_{\inv(j)}$ has full column rank
and $S_j$ is a square matrix of full rank, we can
compute the pseudoinverse of the product
as\footnote{$(A B)^\dagger = B^\dagger A^\dagger$ if $A$ has full
 column rank and $B$ full row rank}  $   \left( Z'_j \cdot \overline{Y'_j} \right)^\dagger =\overline{R^\dagger_{\inv(j)}} S_j^{-1} R^\dagger_{\inv(j)}$
and further for another state $i \in [n]$,
\begin{align}
    \label{eq:7}
    \left( Z'_j \overline{Y'_j} \right)^\dagger \cdot
    \left( Z'_i \overline{Y'_i} \right)
    &= \overline{R^\dagger_{\inv(j)}} S_j^{-1} R^\dagger_{\inv(j)}
    R_{\inv(i)} S_i  \overline{R_{\inv(i)}} .
\end{align}
Assume for now that $i$ is a companion of $j$
such that $\Xi_i$ and $\Xi_j$ are identical.
Then,
$R_{\inv(j)} = R_{\inv(i)}$
and
$R^\dagger_{\inv(j)} R_{\inv(i)} = I_L$
since $R_{\inv(i)}$ has full
column rank.
This is helpful as (\ref{eq:7}) immediately simplifies to
\begin{align}
    \label{eq:8}
    \left( Z'_j \overline{Y'_j} \right)^\dagger \cdot
    \left( Z'_i \overline{Y'_i} \right)
    = \overline{R^\dagger_{\inv(j)}} S_j^{-1} S_i \overline{R_{\inv(i)}}
\end{align}
and we can (almost) obtain $R_{\inv(i)}$ via an
eigendecomposition.
This will be the focus of the following section.
However, it yet remains to detect whether $i$ is
a companion of $j$.
If $i$ and $j$
were companions, we could
write the pseudoinverse of 
$(Z'_j \overline{Y'_j})^\dagger (Z'_i \overline{Y'_i})$
as
\[
    \left( (Z'_j \overline{Y'_j})^\dagger (Z'_i \overline{Y'_i})\right)^\dagger
    = \overline{R^\dagger_{\inv(i)}} S_i^{-1} S_j \overline{R_{\inv(j)}}
\]
which is just 
$(Z'_i \overline{Y'_i})^\dagger (Z'_j \overline{Y'_j})$, as apparent
from (\ref{eq:8}).
It turns out that this property
is also necessary for companionship:

\begin{lemma}
\label{lem:comp}
The states $i$ and $j$ are companions if and only if
$(Z'_j \overline{Y'_j})^\dagger (Z'_i \overline{Y'_i})$
has rank $L$ and
is the pseudoinverse of
$(Z'_i \overline{Y'_i})^\dagger (Z'_j \overline{Y'_j})$.
\end{lemma}

In order to prove Lemma~\ref{lem:comp}, we first
need the following.

\begin{lemma}
\label{lem:inv}
The product $R^\dagger_{\inv(j)} R_{\inv(i)} \in \R^{L \times L}$
is the inverse of $R^\dagger_{\inv(i)} R_{\inv(j)}$ if and only if
$\Xi_i = \Xi_j$, i.e. $i$ and $j$ are connected
in each chain.
\end{lemma}

\begin{proof}
If $\Xi_i = \Xi_j$ then
$R_{\inv(i)} = R_{\inv(j)}$ which of course implies
$R^\dagger_{\inv(j)} R_{\inv(i)} =
R^\dagger_{\inv(i)} R_{\inv(j)} = I_L$.
Conversely, assume that
$\Xi_i \not= \Xi_j$.
Note that the matrices $\Xi_k$ have
rank $L$ as each state $k$ is part of
exactly $L$ different connected components.
The matrix $R$ has full rank, so
$R_{\inv(k)} = R^{-1} \Xi_k \in \R^{r \times L}$
has full column rank $L$.
Combined with the fact that
$R_{\inv(j)} = R^{-1} \Xi_j \not=
R^{-1} \Xi_i = R_{\inv(i)}$
and $R_{\inv(j)}$ and $R_{\inv(i)}$
are both rectangular matrices 
with more columns than rows,
it follows
that $\im(R_{\inv(j)}) \not= \im(R_{\inv(i)})$.
Without loss of generality, there is
an $x \in \im(R_{\inv(j)}) \setminus \im(R_{\inv(i)})$,
otherwise exchange the roles of $i$ and $j$.
We now assume to the contrary that
\[
    R^\dagger_{\inv(j)}
    R_{\inv(i)} R^\dagger_{\inv(i)}
    R_{\inv(j)} = I_L .
\]
We can further use that
$\proj_j \coloneqq R_{\inv(j)} R^\dagger_{\inv(j)}$ and
$\proj_i \coloneqq R_{\inv(i)} R^\dagger_{\inv(i)}$ are
orthogonal projection matrices onto
$\im(R_{\inv(j)})$ and $\im(R_{\inv(i)})$, 
respectively.
Multiplying the above equality by $R_{\inv(j)}$ from the
left and $R^\dagger_{\inv(j)}$ from the right gives
$\proj_j \proj_i \proj_j = \proj_j$ and for our choice of
$x \in \im(R_{\inv(j)})$ that
\[
    \proj_j \proj_i x = \proj_j \proj_i \proj_j x = \proj_j x = x.
\]
We consider the projection into direction
$d \coloneqq \proj_i x - x$
which is orthogonal to $\im(\proj_i)$
and the projection direction
$\proj_j \proj_i x - \proj_i x = x - \proj_i x = -d$
which lies orthogonal to $\im(\proj_j)$.
In particular,
\[
    \langle d, \proj_i x \rangle = 0
    \quad\textrm{and}\quad
    \langle -d, \proj_j x \rangle =
    \langle -d, x \rangle = 0 .
\]
Adding both up, we obtain
$
    0 = \langle d, \proj_i x - x \rangle
    = \langle d, d \rangle
$.
However, $d = \mathbf 0$ implies that
$x \in \im(R_{\inv(i)})$ which is a contradiction.
\end{proof}

We use Lemma~\ref{lem:inv} to prove Lemma~\ref{lem:comp}.

\begin{proof}[Proof of Lemma~\ref{lem:comp}.]
It might be that
$R^\dagger_{\inv(i)} R_{\inv(j)}$ is not
invertible which, by Lemma~\ref{lem:inv},
means that $i$ and $j$ are disconnected in
some chain.
However, note that the matrices 
$\overline{R^\dagger_{\inv(j)}} S_j^{-1}$
and $S_i \overline{R_{\inv(i)}}$ in (\ref{eq:7})
have full column and row-rank
(resp.) of $r \ge L$, which means that the
rank of the product $(Z'_j \overline{Y'_j})^\dagger (Z'_i \overline{Y'_i})$
is determined by 
$R^\dagger_{\inv(i)} R_{\inv(j)}$.
Having rank $L$ is therefore a necessary condition
for companionship.


We may thus assume that
$R^\dagger_{\inv(i)} R_{\inv(j)}$
is invertible, so we can compute
the pseudoinverse
\begin{align*}
    \left( \left( Z'_i \overline{Y'_i} \right)^\dagger \cdot
    \left( Z'_j \overline{Y'_j} \right) \right)^\dagger
    &= \overline{R^\dagger_{\inv(j)}} S_j^{-1}
    \left( R^\dagger_{\inv(i)} R_{\inv(j)} \right)^{-1}
    S_i  \overline{R_{\inv(i)}} .
\end{align*}
Note that this term is identical to
the expression in (\ref{eq:7}) if the product
$R^\dagger_{\inv(j)} R_{\inv(i)}$
is the inverse of $R^\dagger_{\inv(i)} R_{\inv(j)}$.
The latter is even a necessary condition,
as the matrices
$\overline{R^\dagger_{\inv(j)}} S_j^{-1}$
and $S_i \overline{R_{\inv(i)}}$
have full column and row rank, respectively.
It follows by Lemma~\ref{lem:inv} 
that $i$ and $j$ are connected in each chain if
and only if
$(Z'_i \overline{Y'_i})^\dagger (Z'_j \overline{Y'_j})$
is the pseudoinverse of
$(Z'_j \overline{Y'_j})^\dagger (Z'_i \overline{Y'_i})$.
\end{proof}

We can check companionship for
each pair of states $i, j \in [n]$ using Lemma~\ref{lem:comp}
which reveals a set of equivalence classes for
companionship.
We can only obtain the columns of $R^{-1}$
that correspond to connected components
intersecting an individual equivalence class.
We detail this process in the next section,
using a representative state $j$ and a
companion $i$ for each equivalence class.
In section~\ref{subsec:merge}, we then
show how to merge these parts to obtain
the full mixture.

Finally, note that we can avoid
explicit computation of the pseudoinverse
and instead check whether the two matrices
fulfill the defining properties
of a pseudoinverse in a randomized
fashion.


\subsection{Eigendecomposition}
\label{subsec:eigen}

Let us now consider an
equivalence class with representative
state $j$ and companion $i$.
We already pointed out in \eqref{eq:8} that
\begin{align*}
    \left( Z'_j \cdot \overline{Y'_j} \right)^\dagger \cdot
    \left( Z'_i \cdot \overline{Y'_i} \right) &= 
    \overline{R^\dagger_{\inv(j)}} S_j^{-1} R^\dagger_{\inv(j)} \cdot
    R_{\inv(j)} S_i  \overline{R_{\inv(j)}} = \overline{R^\dagger_{\inv(j)}} S_j^{-1} S_i
      \overline{R_{\inv(j)}} .
\end{align*}
By assumption in Theorem~\ref{thm:main}, $S^{-1}_j S_i$ has distinct
diagonal entries, so all eigenspaces
in the above matrix are unidimensional.
We can therefore perform an eigendecomposition
to determine $R_{\inv(j)}$ up to a scaling and
ordering of the eigenvectors. 
In symbols, we obtain a matrix $\tilde R'_j \in \R^{L \times r}$ such that
$\Pi_j D_j \tilde R'_j = R_{\inv(j)}^\dagger \in R^{L \times r}$
for a permutation matrix $\Pi_j \in \R^{L \times L}$ and a diagonal matrix
$D_j \in \R^{L \times L}$.
We fix the ordering arbitrarily and can then undo
the scaling.
To this end, we use a technique from Gupta et al. \cite{gupta2016mixtures},
namely that
we can write the probability to start
a trail in state $i$ and transition to $j$ as
the entries of the vector
$O_j \mathbf 1_n = \overline{P_j} S_j^{-1} Q_j \mathbf 1_n
 = \overline{P_j} \mathbf 1_L \in \R^n$.
Furthermore, $P_j = Y_j P'_j = R_{\inv(j)}^\dagger Y'_j P'_j = \Pi_j D_j \tilde R'_j Y'_j P'_j .$
%
Using $\Pi \cdot \mathbf 1_L = \mathbf 1_L$, this implies 
$  O_j \mathbf 1_n =
    \overline{P_j} \mathbf 1_L =
    \overline{\tilde R'_j Y'_j P'_j}
      \cdot D_j \mathbf 1_L
     \in \R^{n} . $
Since $\tilde R'_j Y'_j P'_j \in \R^{L \times n}$ has full
row rank and $D_j \mathbf 1_L \in \mathbb R^L$,
we can immediately determine $D_j$.
We finally obtain
\[
    \tilde R_j \coloneqq D_j \tilde R'_j = \Pi^{-1}_j R^\dagger_{\inv(j)} \in \R^{L \times r} .
\]
From there, we could compute $R_{\inv(j)}$ which 
is enough to recover transition probabilities
of the $L$ connected components that contain
the state $j$.
In the next section, we show how to merge
$R_{\inv(j)}$ from multiple equivalence classes
to recover the complete mixture.

\subsection{Merging Components}
\label{subsec:merge}

In Section~\ref{subsec:eigen}, we showed
how to obtain the matrix
$\tilde R_j = \Pi^{-1}_j R^\dagger_{\inv(j)}$
for any $j \in [n]$.
The pseudoinverse $\tilde R_j^\dagger
    = R_{\inv(j)} \Pi_j
    = R^{-1} \Xi_j \Pi_j
    \in \R^{r \times L}$
reveals $L$ columns of $R^{-1}$
in arbitrary order (as specified by $\Pi_j$).
Specifically, we obtain
the columns corresponding
to connected components
containing $j$, which is
exactly what $\Xi_j$ indicates.
To obtain all columns of $R^{-1}$,
it is sufficient to combine the columns of
matrices $R^\dagger_j$ for $j \in J$, where $J \subseteq [n]$
is a set of states
such that each of the $r$ connected components
has at least one member in $J$.
As each connected component contains
a representative state, it is
sufficient to combine the columns
of $\tilde R^\dagger_j$ from representative
states $j$.

Having assembled $R^{-1} \Pi^{-1}$ where
$\Pi \in \R^{r \times r}$ is some arbitrary
permutation matrix,
we can use the inverse $(R^{-1} \Pi^{-1})^{-1} = \Pi R$
to reconstruct
\begin{align}
    \label{eq:11}
    \Pi R Y'_j
    = \Pi \Xi_j Y_j \in \R^{r \times L} .
\end{align}
Note that the row in $\Xi_j Y_j$
corresponding to the connected
component $C \in \mathcal C^\ell$ is
either the $\ell$-th row
of $Y_j$ if $j \in C$
or all-zero.
In other words, $\Xi_j Y_j$ contains
$Y_j$, but interspersed with
$r-L$ all-zero rows.
We cannot simply remove the
all-zero rows from $\Pi \Xi_j Y_j$
as this would result in a different
row-permutation for each $Y_j$.
However, we are able to identify
the connected components along with their members
from the non-zero rows (in some arbitrary order
as given by the permutation matrix $\Pi$).
Since we know the $L$ connected components
each representative state is
contained in,
we can assign a label from $[L]$ to all connected
components such that 
each representative state is contained
in exactly one connected component
carrying the label $\ell$, for each $\ell \in [L]$.
This allows us to recover the
matrices $\Pi' Y_j$ where $\Pi' \in \R^{L \times L}$
is another arbitrary permutation matrix:
For each $\ell \in [L]$, we pick
the single non-zero row from
$\Pi \Xi_j Y_j$ that corresponds
to a connected component with label $\ell$.

Similarly, we obtain $\Pi' Z_j$.
With access to $\Pi' Y_j$ and $\Pi' Z_j$ for each
$j \in [n]$, we can recover the full mixture,
up to an arbitrary permutation $\Pi'$.
We first determine the starting probabilities
\[
    \Pi' Z_j \overline{\Pi' Y_j}
    =  \Pi' Z_j \overline{Y_j} \overline{\Pi'}
    =  \Pi' S_j \overline{X_j} \overline{X_j^{-1}} \overline{\Pi'}
    =  \Pi' S_j \overline{\Pi'}
\]
as well as $\Pi' Y_j P'_j = \Pi' X_j^{-1} P'_j = \Pi' P_j .$
The matrices $P_j$ contain the values
$s_i \cdot M^\ell_{ij}$, so we can recover the
transition probabilities $M^\ell_{ij}$
by dividing by $s_i$ and we are done.

\subsection{Running Time}
The asymptotic running time of our algorithm is
$O(n^5 + n^3 L^3 + L^r)$
where $r$ is the total number of connected components
in the mixture, in the context of Theorem~\ref{thm:main}.
Two steps determine this complexity:
First, we need $O(n^5 + n^3 L^3)$
to compute the co-kernel of $\A$ either
through Gaussian elimination or
SVD.\footnote{Computing the SVD of an $m \times n$ matrix
takes time $O(m n^2)$.}
We could speed this step up
by solving the SVD approximately, and it is an
interesting direction to analyze the effect
of the resulting noise on the algorithm.
The other steps involving SVD and
eigendecomposition are asymptotically faster and
easily parallelizable.
Second, finding a proper assignment
$a \colon [r] \to [L]$
to merge components is NP-complete,
so $O(L^r)$ is necessary.

\section{Choosing the parameters}
\label{subsec:degen}


So far, we have assumed knowledge of the two variables $L$ and $r$, but in real-world scenarios they are frequently latent. How can we learn them from the data?
\spara{Number of Chains $L$} Recall that
$O_j = \overline{P_j} \Mout_j$
where we define
$\Mout_j \in \R^{L \times n}$
as the matrix of transition probabilities
out of $j$, i.e.
$\Mout_j(\ell, k) \coloneqq M^\ell_{ik}$ for
all $\ell \in [L]$ and $k \in [n]$.
We should thus be able to derive $L$ from the
rank of $O_j$, assuming that the
matrices $P_j \in \R^{L \times n}$ and
$\Mout_j \in \R^{L \times n}$ have
full rank $L$.
However, simply computing the
rank might be error-prone
due to noise in the 3-trail distribution
or degeneracies in the underlying mixture.
We should rather look at the magnitude of
singular values of $O_j$ to get a more robust estimate.
The following result establishes a relationship
between the $L$-th largest singular
value of $O_j$ and the $L$-th largest singular values of
transition matrices $P_j$ and $\Mout_j$.
\begin{theorem}
    \label{thm:degen}
    For all states $j \in [n]$  ,
    \[
        \sigma_L(P_j) \cdot \sigma_L(\Mout_j) \le \sigma_L(O_j) 
        \le \sqrt L \min\left\{\sigma_L(P_j), \sigma_L(\Mout_j)\right\} .
    \]
\end{theorem}

The proof of Theorem~\ref{thm:degen} is standard
as it relies mainly on the decomposition
$O_j = P_j \Mout_j$ and the resulting relationship
between the smallest (i.e., $L$-th largest)
singular values of $P_j$ and $\Mout_j$ and the
$L$-th largest singular value of $O_j$. We include it
here for completeness.

\begin{proof}
    First, note that $Q_j = S_j \Mout_j$ and thus
    \[
        O_j = \overline{P_j} S^{-1} Q_j
        = \overline{P_j} \Mout_j .
    \]
    We begin by showing the first inequality.
    We can write the $L$-th largest
    singular value as
    \begin{align*}
        \sigma_L(\overline{P_j} \Mout_j)
        &= \max_{\dim(U) = L} \min_{u \in U} \frac{\| \overline{P_j} \Mout_j u \|_2}{\| u \|_2} \\
        &= \max_{\dim(U) = L} \min_{u \in U} \frac{\|\Mout_j u\|_2}{\| u \|_2} \cdot \frac{\| \overline{P_j} \Mout_j u \|_2}{\|\Mout_j u\|_2}
    \end{align*}
    where $U$ is any subspace of $\R^n$.
    We bound the above by splitting the product
    into
    \begin{align*}
        \sigma_L(\overline{P_j} \Mout_j)
        &\ge \max_{\dim(U) = L} \left( \min_{u \in U} \frac{\|\Mout_j u\|_2}{\| u \|_2} \right) \cdot \left( \min_{u \in U} \frac{\| \overline{P_j} \Mout_j u \|_2}{\|\Mout_j u\|_2}\right) \\
        &\ge \max_{\dim(U) = L} \left( \min_{u \in U} \frac{\|\Mout_j u\|_2}{\| u \|_2} \right) \cdot \left( \min_{w \in \R^L} \frac{\| \overline{P_j} w \|_2}{\|w\|_2}\right)
    \end{align*}
    since $\Mout_j v \in \R^L$.
    Note that the second term is just the smallest (or $L$-th)
    singular value of $P_j$.
    Hence,
    \begin{multline*}
        \sigma_L(\overline{P_j} \Mout_j)
        \ge \max_{\dim(U) = L} \left( \min_{u \in U} \frac{\|\Mout_j u\|_2}{\| u \|_2} \right) \cdot \sigma_L(P_j) = \sigma_L(P_j) \cdot \max_{\dim(U) = L} \min_{u \in U} \frac{\|\Mout_j u\|_2}{\| u \|_2} \\ = \sigma_L(P_j) \cdot \sigma_L(\Mout_j)
    \end{multline*}
    where again $U$ is any subspace of $\R^n$. For the second inequality, we upper bound $ \| \overline{P_j} \Mout_j u \|_2 $ as follows 
    \begin{align*}
        \| \overline{P_j} \Mout_j u \|_2 
        \le  \| P_j \|_2 \cdot \| \Mout_j u \|_2 .
    \end{align*}
Since $P_j \in [0,1]^{L \times n}$,
    the operator norm $\| P_j \|_2$ is maximized
    when all entries of $P_j$ are 1.
    As such, $\| P_j \|_2 \le \sqrt L$.
    We can now bound the $L$-th largest singular
    value by
    \begin{align*}
        \sigma_L(\overline{P_j} \Mout_j)
        &= \max_{\dim(U) = L} \min_{u \in U} \frac{\| \overline{P_j} \Mout_j u \|_2}{\| u \|_2}  \le \| P_j \|_2 \cdot \max_{\dim(U) = L} \min_{u \in U} \frac{\| \Mout_j u \|_2}{\| u \|_2} = \sqrt L \cdot \sigma_L(\Mout_j) .
    \end{align*}
    Similarly, we can obtain
    $\sigma_L(\overline{P_j} \Mout_j) \le \sqrt L \cdot \sigma_L(P_j)$
    by considering the transposed matrix $\overline{\Mout_j} P_j$.
\end{proof}

Assume that the $L$-th singular values of
$P_j$ and $\Mout_j$ for some $j$ are high. This
indicates that the transition probabilities in
and out of $j$ are very dissimilar and
cannot be represented with less than $L$ chains.
On the other hand, having two near-identical chains
in the mixture will induce a low $L$-th
singular value and we should therefore
reduce $L$ to avoid overfitting.
Both ideas allow us to estimate
$L$ robustly: We set $L$ such that
for most of the states $j \in [n]$,
the first $L$ singular values
are large and the remaining singular
values are small.

\spara{Number of Connected Components} If the mixture satisfies the conditions of Theorem~\ref{thm:main}, in particular
if the co-kernel of $\A$ is spanned by
indicator vectors for the connected components,
we can deduce $r$ from the rank of $\A'$.
However, as with estimating $L$, this is
prone to error and we should rather look
at the singular values of $\A'$.
In the following, we show how to derive
an upper bound on the $(r+1)$-th smallest
singular value $\sigma_{2Ln -r}$ that shows
a necessary condition for the robust
estimation of $r$.
A lower bound is much harder to attain
as it would require a complete combinatorial
characterization of the co-kernel of $\A$ which
seems difficult to obtain.
Empirically, we do not have troubles in estimating
$r$.

For a cut
$\emptyset \subsetneq S \subsetneq V^\pm$,
we define $\Mcut_S \in \R^{L \times |V^\pm|^2}$
as the matrix of all
transition probabilities
across the cut.
We set
$\Mcut_S(\ell, (i^-, j^+)) \coloneqq s^\ell_i M^\ell_{i, j}$
if $i^-$ and $j^+$ are 
on different sides of $S$ and 
$\Mcut_S(\ell, (x, y)) \coloneqq 0$
for all remaining $x, y \in V^\pm$.
Let now $S$ be a cut such that
the smallest singular value
\begin{align}
    \label{eq:9}
    \sigma_L(\Mcut_S) = \min_{w \in \R^L} \frac{\| w^\top \Mcut_S \|_2}{\| w \|_2}
\end{align}
is small, implying
that there is high similarity in
the chains' transition probabilities
across the cut.
Recall that the $r$ smallest singular values
$\sigma_{2Ln - r+1}(\A)$, $\sigma_{2Ln-r+2}(\A)$, $\dots, \sigma_{2Ln}(\A)$
are all zero, as
$r$ is exactly the dimension of the
co-kernel of $\A$.
The similarity between different chains
as measured by $\sigma_{\min}(\Mcut_S)$
relates to the 
$(2Ln - r)$-th singular value of $\A$
in the following way.

\begin{theorem}
    \label{thm:sigma}
    For any cut $\emptyset \subsetneq S \subsetneq V^\pm$, $\sigma_{2Ln - r}(\A) \le \frac 1 {|S|} \sigma_L(\Mcut_S) .$
\end{theorem}

\begin{proof}
    Let $w \in \R^L$ be the minimizer of the
    quotient in (\ref{eq:9}) with length
    $\|w\|_2 = 1$.
    Let $v \in \R^{2Ln}$ where
    $v_x \coloneqq w$ for $x \in S$ and
    $v_x \coloneqq 0$,
    otherwise.
    Let $K \subseteq \R^{2Ln}$ be the $(r+1)$-dimensional
    vector space
    spanned by $v$ and the indicator
    vectors $\xi_1, \xi_2, \dots, \xi_r$.
    We can bound the $(2Ln-r)$-th
    smallest singular value of $\A$
    by
    \begin{align}
        \sigma^2_{2Ln-r}(\A)
        =
        \min_{\dim(U) = r+1} \max_{u \in U} \frac{\| u^\top \A \|^2_2}{\| u \|^2_2}
        \le
        \max_{u \in K} \frac{\| u^\top \A \|^2_2}{\| u \|^2_2}
        =
        \max_{u \in K} \frac 1 {\| u \|^2_2}
            \sum_{i, j} \left((u_{j^+} - u_{i^-})^\top S_i M_{ij}\right)^2
        \label{eq:4}
    \end{align}
    where the inequality holds because $K$ is also $(r+1)$-dimensional
    and the final equality is by definition of $\A$.
    We now show that for
    $u = \lambda_0 v + \lambda_1 \xi_1 + \cdots + \lambda_r \xi_r \in K$
    and any fixed pair of vertices $x, y \in V^\pm$,
    \begin{align}
        \label{eq:10}
        (u_x - u_y)^\top S_i M_{ij}
        = \lambda_0 (v_x - v_y)^\top S_i M_{ij}.
    \end{align}
    Fix any $\xi_i$ and let $\ell$ be the chain
    containing the connected component $C \in \mathcal C^\ell$
    that corresponds to $\xi_i$.
    Either $x$ and $y$ are both on the same side of $C$ (i),
    or one of $x$ and $y$ is in $C$ and the
    other one is not (ii).
    In case of (i), $(\xi^\ell_C)_x = (\xi^\ell_C)_y$
    and $\lambda_i \xi_i$ cancels out in $u_x - u_y$.
    In (ii), $M^\ell_{ij} = 0$
    as $x$ and $y$ are not connected in $M^\ell$.
    Consequently, the $\ell$-th entry of
    $u_x - u_y$ vanishes.
    Overall, only $\lambda_0 v$ remains of $u$ as
    reflected in the RHS of (\ref{eq:10}).
    It follows that $v$ maximizes the final term in
    (\ref{eq:4}) and thus
    \begin{align*}
        \sigma^2_{2Ln-r}(\A)
        &\le
        \frac 1 {\| v \|^2_2} 
            \sum_{i, j} \left((v_{j^+} - v_{i^-})^\top S_i M_{ij}\right)^2 \\
        &= 
        \frac 1 {\| v \|^2_2} \left(
            \sum_{i^- \in S, j^+ \notin S} (- w^\top S_i M_{ij})^2
        +
            \sum_{i^- \notin S, j^+ \in S} (w^\top S_i M_{ij})^2
            \right) \\
        &= \frac 1 {\| v \|^2_2} \| w^\top \Mcut_S \|_2^2
        = \frac 1 {\| v \|^2_2} \sigma_L(\Mcut_S).
    \end{align*}
    where the first equality is by definition of $v$.
    Finally, $v$ has length $\| v \|_2^2 = | S | \cdot \| w \|^2_2 = | S |$,
    so $\sigma^2_{2Ln-r}(\A) \le \frac 1 {|S|} \sigma_L(\Mcut_S)$.
\end{proof}

We can express the bound more naturally through
the similarity of any pair of chains.

\begin{corollary}
    \label{cor:tv}
    We have $\sigma^2_{2Ln-r}(\A) \le \TV(M^\ell, M^{\ell'})$
    for any two chains $\ell \not= \ell'$ in the mixture.
\end{corollary}

\begin{proof}
    We simply calculate
    \begin{align*}
        \TV(M^\ell, M^{\ell'})
        =
        \frac 1 {2n} \sum_{i, j} \left| M_{ij}^\ell - M^{\ell'}_{ij} \right|
        =
        \frac 1 {2n} \sum_i \left\| (\mathbf 1_\ell - \mathbf 1_{\ell'})^\top \Mout_i \right\|_1
    \end{align*}
    We can lower bound each $\| \cdot \|_1$ 
    in the sum by $\| \cdot \|_2^2$ since the entries in the vector
    $(\mathbf 1_\ell - \mathbf 1_{\ell'})^\top \Mout_i$ are at most $1$.
    Hence,
    \begin{align*}
        \TV(M^\ell, M^{\ell'})
        \ge
        \frac 1 {2n} \sum_i \left\| (\mathbf 1_\ell - \mathbf 1_{\ell'})^\top \Mout_i \right\|^2_2
        =
        \frac 1 n \sum_i \left\| \frac 1 {\sqrt 2} (\mathbf 1_\ell - \mathbf 1_{\ell'})^\top \Mout_i \right\|^2_2
        \ge
        \frac 1 n \sum_i \min_{\|w\|_2 = 1} \left\| w^\top \Mout_i \right\|^2_2
    \end{align*}
    where the last inequality is due to $\| \frac 1 {\sqrt 2} (\mathbf 1_\ell - \mathbf 1_{\ell'}) \|_2 = 1$.
    All entries in $\Mout_i$ are positive, so we can bound
    \[
        \| w^\top \Mout_i \|^2_2 \ge \| w^\top S_i \Mout_i \|^2_2
        = \| w^\top \Mcut_{\{i^-\}} \|^2_2 .
    \]
    By Theorem~(\ref{thm:sigma}),
    \[
        \min_{\|w\|_2 = 1} \| w^\top \Mcut_{\{i^-\}} \|^2_2
        = \sigma^2_L(\Mcut_{\{i^-\}})
        \ge \sigma^2_{2Ln-r}(\A)
    \]
    which directly implies what we wanted to prove.
\end{proof}

\section{Experimental Results} 
\label{sec:exp}

\def\arraystretch{0.5}

\subsection{Experimental Setup} 
\label{subsec:setup}

\spara{Datasets} We use both real and synthetic data in our experiments. 
To generate synthetic mixtures that allow us to compare methods knowing the groundtruth, we take three arguments as input: (i) the total number of states $n$, (ii) the number of chains $L$ and (iii) the total number of  weakly connected components across all chains $r$. Notice that when $r=L$, there exists exactly one connected component per chain, meaning that all chains are (weakly) connected.  To create a mixture, we sample uniform random stochastic matrices for the transition probabilities
$M^\ell$ of all chains $\ell \in [L]$. The vector of all starting probabilities
$(s^\ell_k)_{\ell, k} \in \R^{Ln}$ is also sampled uniformly at random from all
non-negative vectors summing to $1$. If $r > L$, we need to create $r-L$ additional connected components, which we achieve this by splitting some chains: First, we assign each of the $r-L$ splits to the $L$ chains uniformly at random. Second, having determined the number of connected components in each chain $\ell$, we assign each state randomly to a connected component. We then disconnect each newly formed component
$C \in \mathcal C^\ell$ from the rest of the graph by setting $M^\ell_{ij} = 0$ for all
$i \in C, j \notin C$ and $i \notin C, j \in C$, and normalize such that $M^\ell$
remains stochastic. We then sample a sequence of trails from this mixture to obtain the 3-trail
distribution.
 
We have used various real-world datasets from the UCI Machine Learning Repository. The {\it MSNBC} dataset contains information about page view sequences for almost one million user sessions. Page views are classified into 17 categories. We cut each sequence into 3-trails and try to learn the underlying mixture for different values of $L$. 
The {\it mushroom} datasets contains 8000 mushroom samples, each associated with 22 categorical variables (features). Since $\expm$ does not scale both in terms of $n$ and $L$ we choose a random subset of  $n=30$ mushroom samples  and 10 features. We add a trail $(i, j, k)$ whenever the mushroom samples $i$, $j$, and $k$ have the same expression for a feature. Naturally, the mixture consists of $L=10$ chains, one per feature. Each categorical value of each feature   induces a fully connected component in the chain.

\noindent \spara{Evaluation metrics}  We use two measures to evaluate the performance
of a recovery algorithm, both based on the total variation (TV) distance. The TV distance of two probability distributions $p$ and $q$ on a discrete sample space  $\Omega$ is half their $\ell_1$ distance, i.e., 
$TV(p, q) = \frac 1 2 \sum_{\omega \in \Omega} |p(\omega) - q(\omega)|$. 
The first metric measures the distance between the ground-truth mixture $\M$ and the 
learned mixture $\mathcal N$.  Since a mixture contains $L>1$ chains, we first find a good alignment between the $L$ chains of the two mixtures through a permutation $\sigma \in S_L$ that matches the chains in $\mathcal N$
to the chains $\mathcal M$, such that the sum of their TV-distances are minimal (we can find $\sigma$ efficiently, e.g. by using the Hungarian algorithm~\cite{korte2011combinatorial}). Once we find $\sigma$, we simply compute the sum of TV-distances over all $L$ pairs of chains. Formally, we define the \textrm{recovery-error}  of $\mathcal N$ for the groundtruth mixture $\mathcal M$ as 
\begin{equation}
\label{eq:re}
    \textrm{recovery-error}(\M, \mathcal N) \coloneqq
    \frac 1 {2Ln} \min_{\sigma \in S_L} \sum_{\ell=1}^L
        \sum_{i=1}^n TV(M^\ell_i, N^{\sigma(\ell)}_i).
\end{equation}
For real-world instances, $\mathcal M$ is typically not known and the Markovian assumption may not be entirely true~\cite{chierichetti2012web}. As a proxy to the recovery error, we compute the TV-distance of the given and learned 3-trail distributions, which we call the \emph{trail-error}.

\noindent \spara{Competitors}  We compare our reconstruction algorithm $\casvd$ to the GKV paper ($\gkvsvd$ algorithm) that is the main competitor. We also implement the   expectation maximization $\expm$ algorithm that overall performs well in terms of quality but is very slow. 

\noindent \spara{Implementation notes}
We adapt the implementation to the presence of noise.
First, we compute how far the products of Lemma~\ref{lem:comp}
are from being pseudoinverses and cluster states accordingly
to determine the companionship equivalence classes.
To mitigate the effect of potential mistakes in the clustering, we
repeat the selection of representative states and their companions 5 times.
Second, since rows of the matrices in \eqref{eq:11} are no
longer exactly zero, we use an IP solver to find the
assignment with least overlap.

\noindent \spara{Machine specs} Our code was implemented in Python~3 and is available online.\footnote{\url{https://github.com/285714/WWW2023/blob/master/experiments.ipynb}}
The experiments were run on a 2.9 GHz Intel Xeon Gold 6226R processor with 384GB of RAM. 

\subsection{Synthetic Experiments}

The probability of a 3-trail $i \to j \to k$ can be   computed  using the equation $p(i, j, k) \coloneqq    \sum_{\ell=1}^L s^\ell_i \cdot M^\ell_{ij} \cdot M^\ell_{jk}$ when the groundtruth mixture is known. In a real-world setting we can only estimate it from data. It can be shown using Chernoff bounds that $O(n^3 \log n/\epsilon^2)$ samples suffice to estimate this probability within $\pm \epsilon$ for all states $j$. However, in many real-world scenarios we do not have enough samples or they are expensive to generate. We explore the effect of the noise on the 3-trail probability distribution.

\spara{Learning with the exact 3-trail distribution}  Figure~\ref{fig:puvn}(c) shows results when given the exact 3-trail distribution from the groundtruth mixture. When $r=L$, both $\gkvsvd$ and our method $\casvd$ can perfectly recover the mixture, and EM performs well achieving low error for both the trail distributions and the recovery error after 100 iterations, which was the upper bound we set in our code (EM already takes considerably longer, cf. Figure~\ref{fig:scal}). It is likely that with more iterations it would also find the groundtruth mixture.  However, once $r>L$ which implies that at least one of the chains is disconnected (i.e., $r=4$), the performance of $\gkvsvd$ and EM quickly deteriorate.  As the lack of connectivity increases, i.e., $r$ grows, their performance drops further.

\begin{figure}
    \centering
    \begin{tabular}{ccc}
      \includegraphics[width=0.3\linewidth]{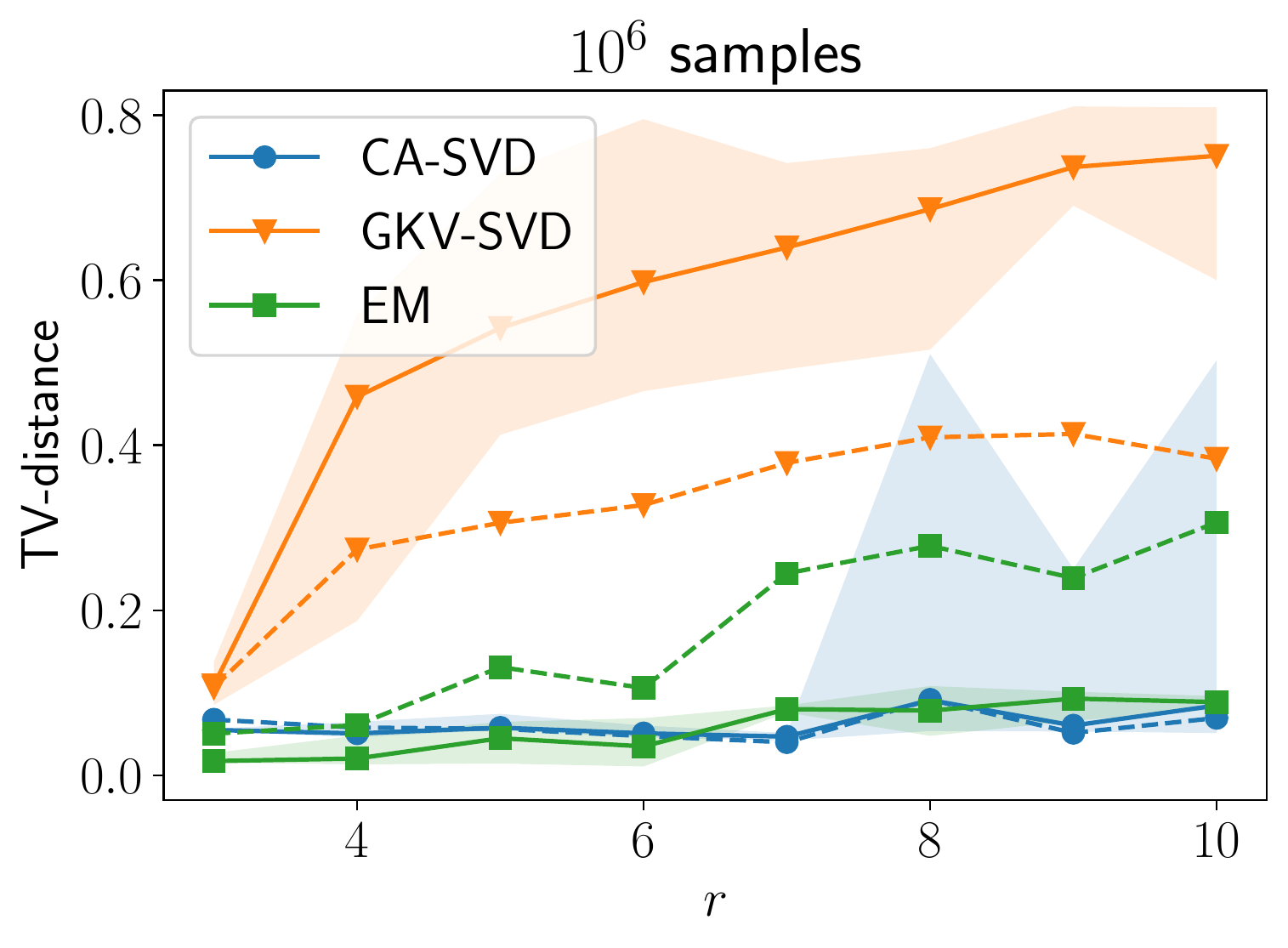} &
      \includegraphics[width=0.3\linewidth]{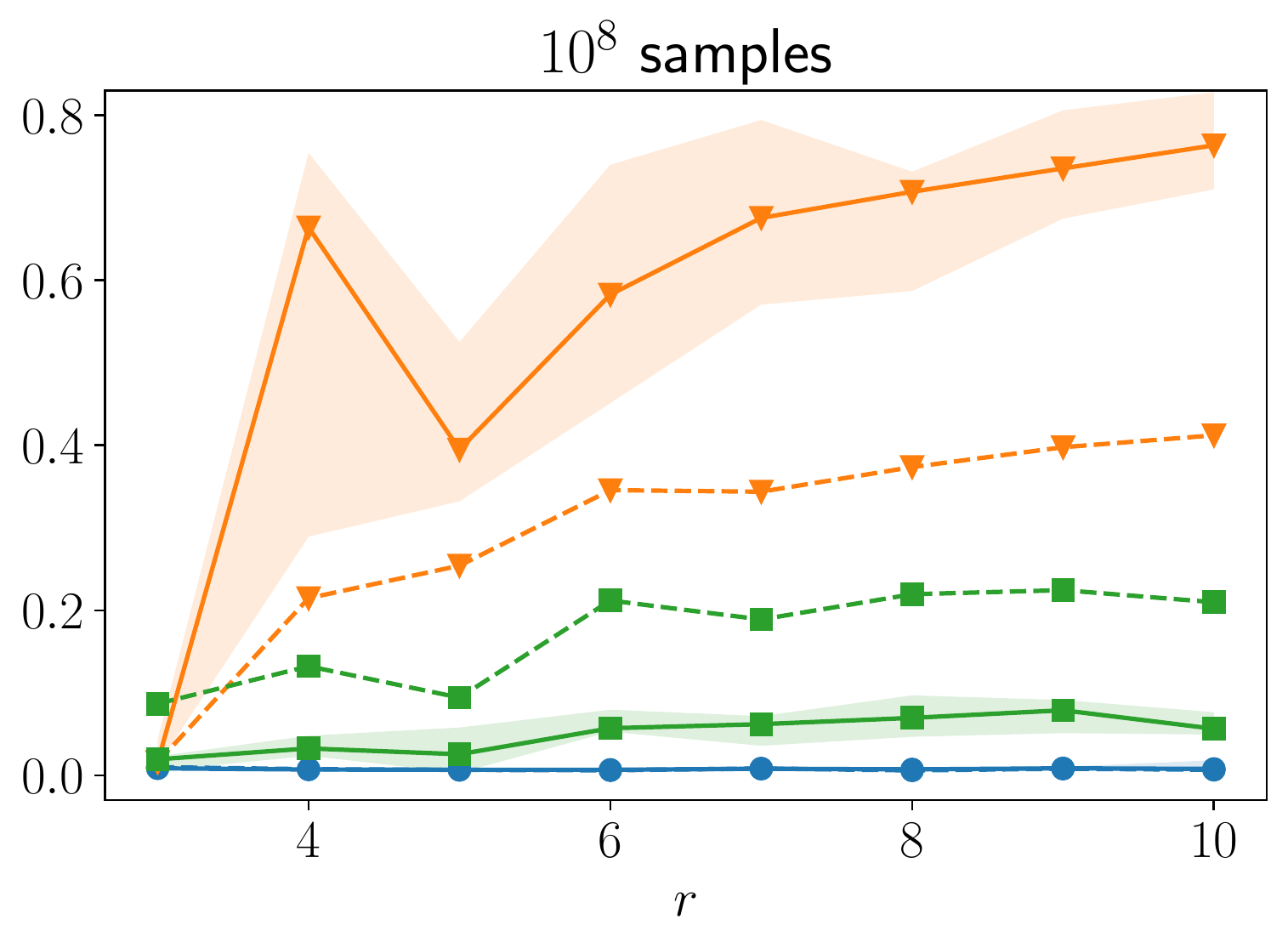} &
      \includegraphics[width=0.3\linewidth]{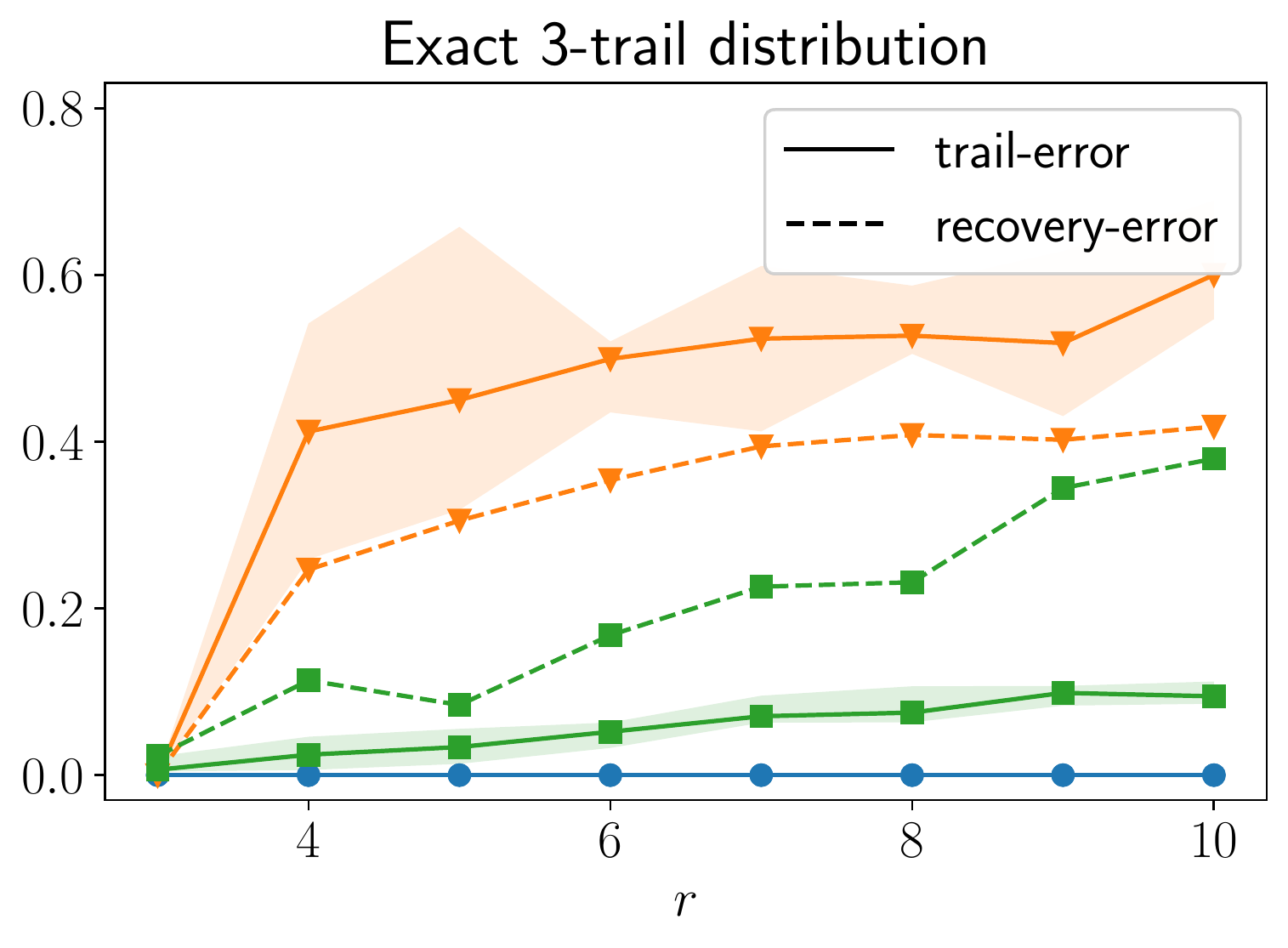} \\ 
         (a) & (b) & (c)
    \end{tabular}
    \negfigsp
     \caption{\label{fig:puvn} 
     Performance under $10^6$ samples (a), $10^8$ samples, and the exact 3-trail distribution from groundtruth (c) on a mixture of $L=3$ chains and $n=20$ states.  Solid lines show the median trail-error of each method,  and the shaded region around each line contains the 25th to 75th percentile. Dashed lines show the recovery-error (Equation~\eqref{eq:re}).
     }
 \end{figure}

\spara{Learning  with the empirical 3-trail distribution} We explore the effect of noise in learning the mixture as follows: We generate  random mixtures with $L=3$ and $n=20$.  We vary $r$, the total number of connected components. We sample $10^6$ and $10^8$ trails that constitute to two levels of noise (more samples correspond to less noise). We repeat each experiment 10 times.

Figures~\ref{fig:puvn}(a),~(b) show our findings when we have access to $10^6$ and $10^8$ 3-trail samples. We observe that for $10^8$ samples, we essentially obtain the same performance as if there were no noise. For $10^6$ samples, the error of the 3-trail distribution affects the $\gkvsvd$ more than $\expm$ and our method $\casvd$. Furthermore, both $\gkvsvd$ and $\expm$ are not well suited when $r > L$, i.e., there is lack of connectivity in at least one of the groundtruth chains. Interestingly, our reconstruction algorithm is robust to noise and increasing lack of connectivity. 

\noindent \spara{Sample complexity} 
For each method, we plot
the required number of samples to achieve TV-distance (for trail-error and recovery error) below $0.05$ in Figure\ref{fig:sc}(a). Each experiment is repeated 10 times. Since for $r > 3$, the matrix $\A$ is no
longer well-distributed, so $\gkvsvd$ is unable to hit this threshold for any number of samples.  Increasing $r$ leaves our method $\casvd$ mostly unaffected, while as before, $\expm$ performs worse for large $r$. We additionally plot the sample complexity of expectation maximization restricted to only $50$ and 100 iterations, which we refer to as $\expm^{50}$ and $\expm^{100}$. This version performs considerably worse than $\expm$.


\begin{figure}
 \centering
\begin{tabular}{cc}
 \includegraphics[width=0.45\linewidth]{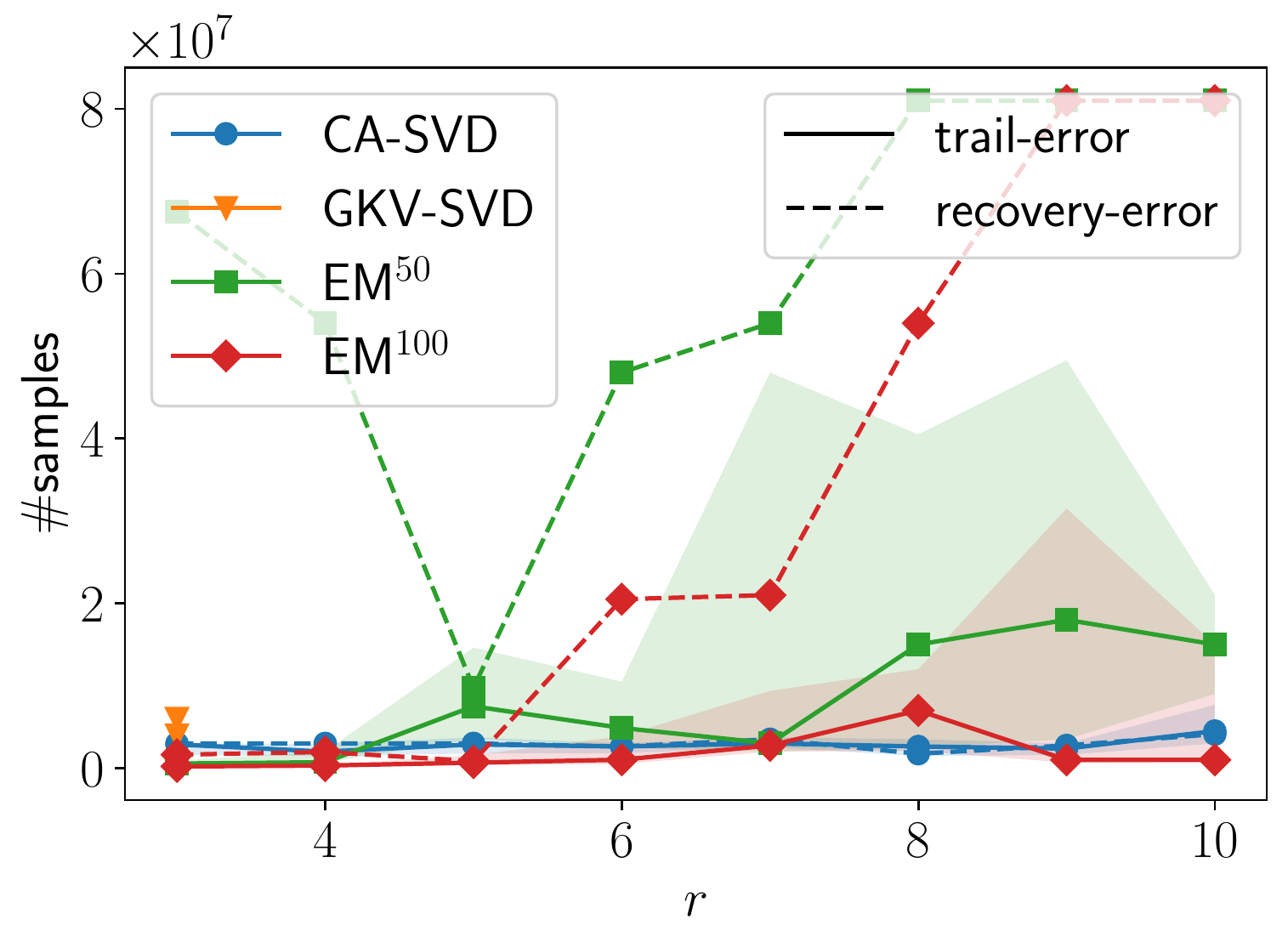} &
 \includegraphics[width=0.45\linewidth]{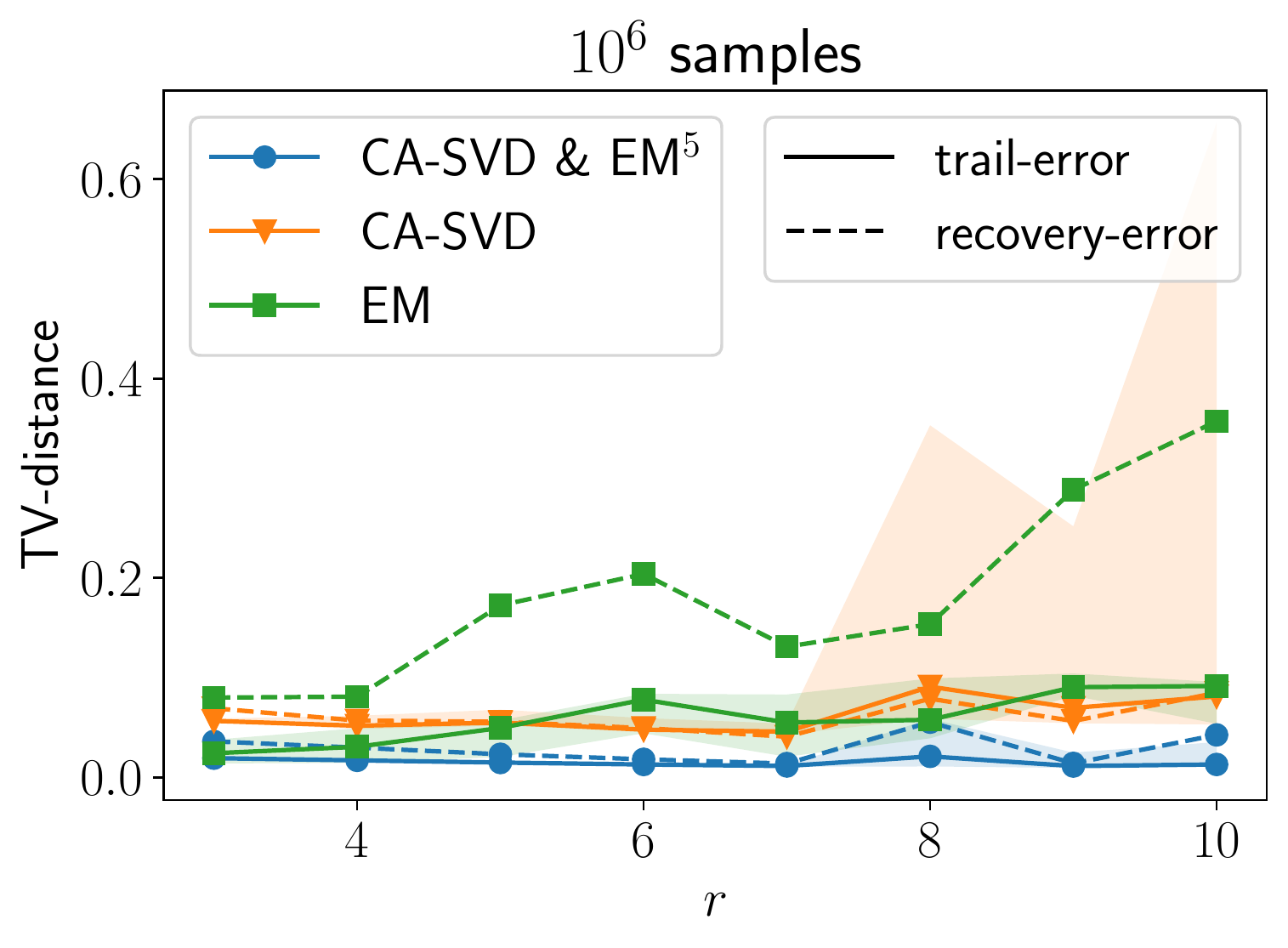} \\
 (a) & (b) 
 \end{tabular}
 \negfigsp
   \caption{\label{fig:sc}  (a) Sample complexity 
   for varying number of connected components $r$ on a mixture with $L=3$ chains
   and $n=20$ states.  (b) Total variation distance versus $r$ for $\expm$, $\casvd$ and their combination.} 
 \end{figure}

\noindent \spara{Combining $\casvd$ with $\expm$} We explore what happens when we first run our recovery method and use the output as starting point for EM. Figure~\ref{fig:sc}(b) shows
that the combination far outperforms the two individual approaches. Furthermore, refining with $\expm$ for only a few iterations
(we use 5 iterations in our experiments) is sufficient.


\noindent \spara{Scalability} In order to be practically applicable, our method needs to be performant even for large values of $n$, $L$, and $r$. As pointed out by Gupta et al. \cite{gupta2016mixtures}, $\gkvsvd$ vastly outperforms $\expm$ for
increasing $n$ and $L$. We replicate these results in Figure~\ref{fig:scal}(a), where we can clearly see that
the running time of $\expm$ scales quadratically with $n$, while the number of iterations until convergence is not affected
by $n$. The additional computation needed for $\casvd$ over $\gkvsvd$ does not add substantial running time and is partly due to overhead from the solver initialization.

Figure~\ref{fig:scal}(b) shows the
running time as a function of the number of
connected components $r$.
Running times for $\casvd$ and $\gkvsvd$
are mostly unaffected from increasing $r$,
but we can see that $\expm$ needs more
iterations and a higher running time
to handle the additional complexity
of disconnected chains.

\begin{figure}[htbp]
    \centering
    \begin{tabular}{cc}
        \includegraphics[width=0.45\linewidth]{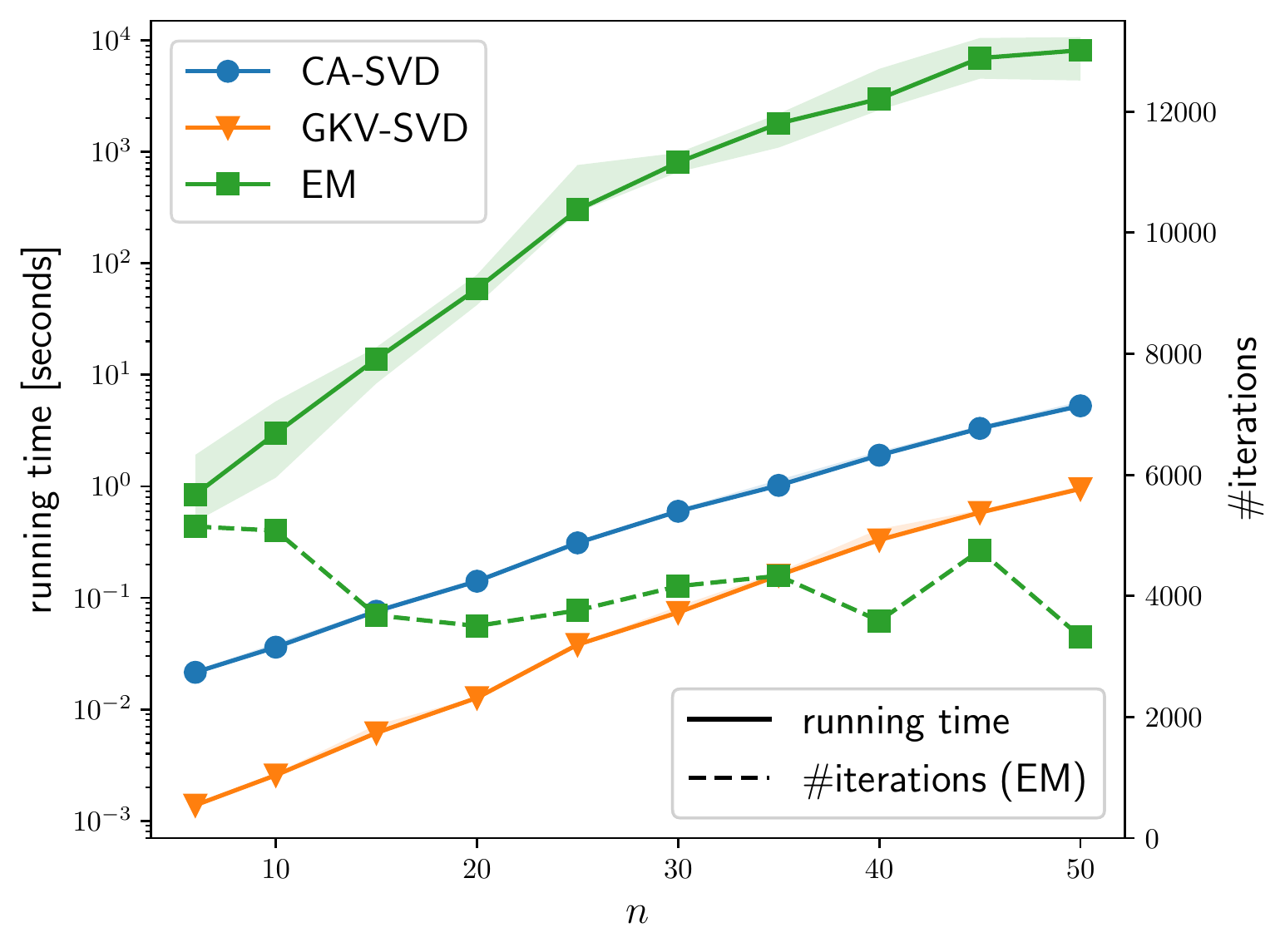} &
        \includegraphics[width=0.45\linewidth]{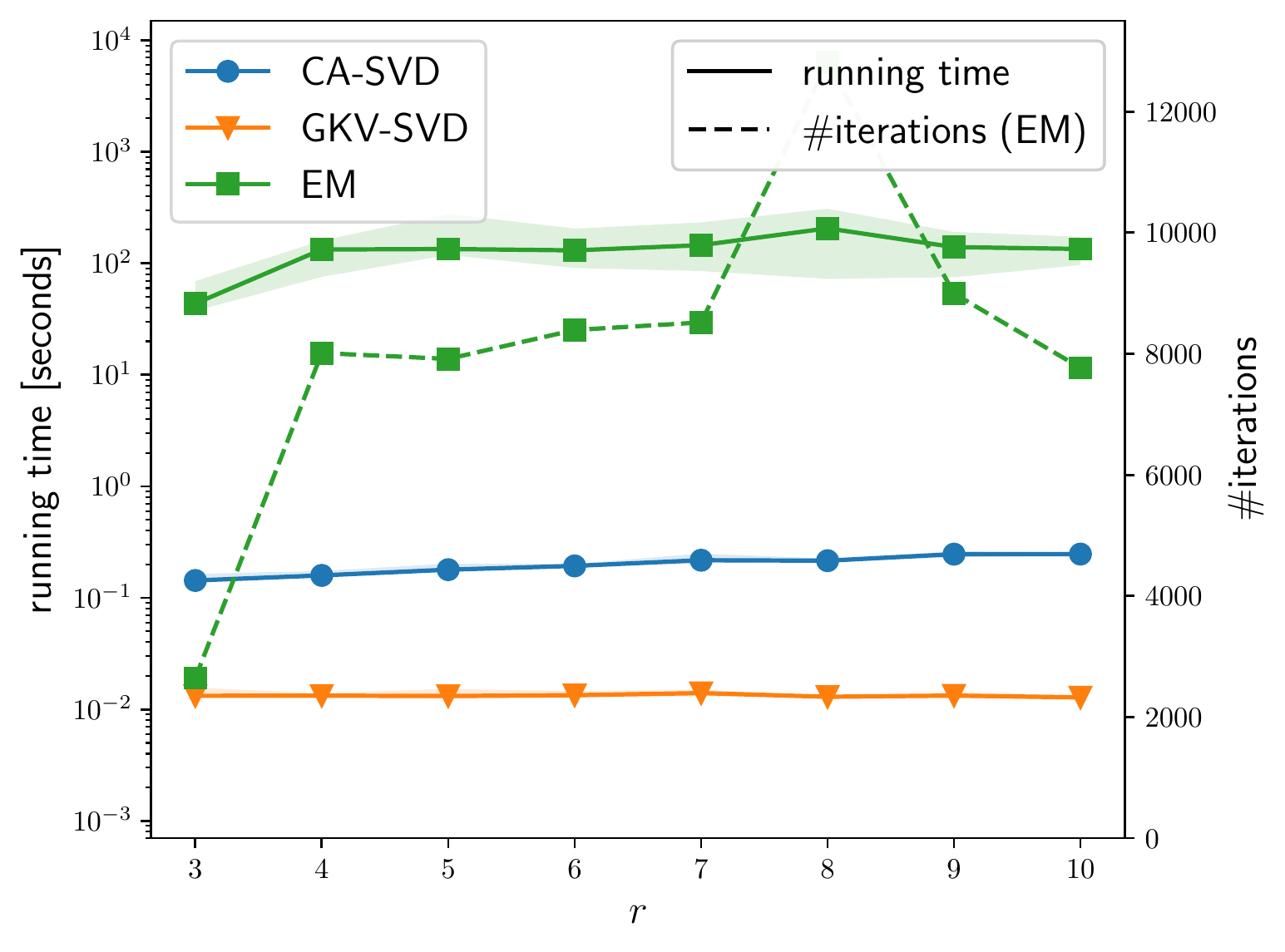} \\
        (a) & (b)
    \end{tabular}
    \negfigsp
    \caption{\label{fig:scal} Plot (a) shows the running time
    (solid lines; the shaded area contains the 25th to 75th percentile)
    and number of iterations until convergence for EM (dashed line)
    for varying $n$ on $L=3$ chains.
    Plot (b) shows the same for
    varying $r$ and a fixed number of $n=20$ states. }
\end{figure}

\noindent \spara{Choosing the number of chains $L$}  We show how singular values can guide the choice of $L$ even when there is a strong form of degeneracy. To this end, we take a random mixture $\M$ and make one chain (say $M^1$)  more similar to another chain (say $M^2$).
For a similarity parameter  $\lambda \in [0,1]$, we update $M^1 \gets (1-\lambda) M^1 + \lambda M^2$.   This means that for $\lambda=0$, we retain  the original chains;
for $\lambda=1$, the first and second chain are identical.  Figure~\ref{fig:sval} shows  information about the average $i$-th largest singular value  $\bar \sigma_i \coloneqq \frac 1 n \sum_{j=1}^n \sigma_i(O_j)$  for some $i \in [n]$  that are interesting in the context of Theorem~\ref{thm:degen} in two scenarios with increasing degeneracy.  As a reference point for degeneracy, we also show the $L$-th largest (i.e. the smallest) singular value  of the transition matrices $P_j$ and $M^+_j$, i.e.
\begin{equation} 
\label{eq:sigmamin}
\sigma_{\min} \coloneqq \min_{j} \min \{\sigma_{\min}(P_j), \sigma_{\min}(M^+_j)\} .
\end{equation}

Figure~\ref{fig:sval}(a) shows  how $\sigma_L$ remains large for $\lambda < 1$ compared to $\sigma_{L+1}$, as promised by the lower bound  in Theorem~\ref{thm:degen}.  As $\lambda$ approaches $1$ and $\sigma_{\min}$ drops to $0$, we further observe that the upper bound
in Theorem~\ref{thm:degen} forces $\sigma_L$ down, but  $\sigma_{L+1}$ remains large.
We can best express this through the ratio of  consecutive singular values as shown in
Figure~\ref{fig:sval}(b). If singular value $i$  has maximum ratio $\sigma_i / \sigma_{i+1}$,
this indicates that we should set $L = i$.
\begin{figure}[htbp]
    \centering
    \begin{tabular}{cccc}
        \includegraphics[width=0.23\linewidth]{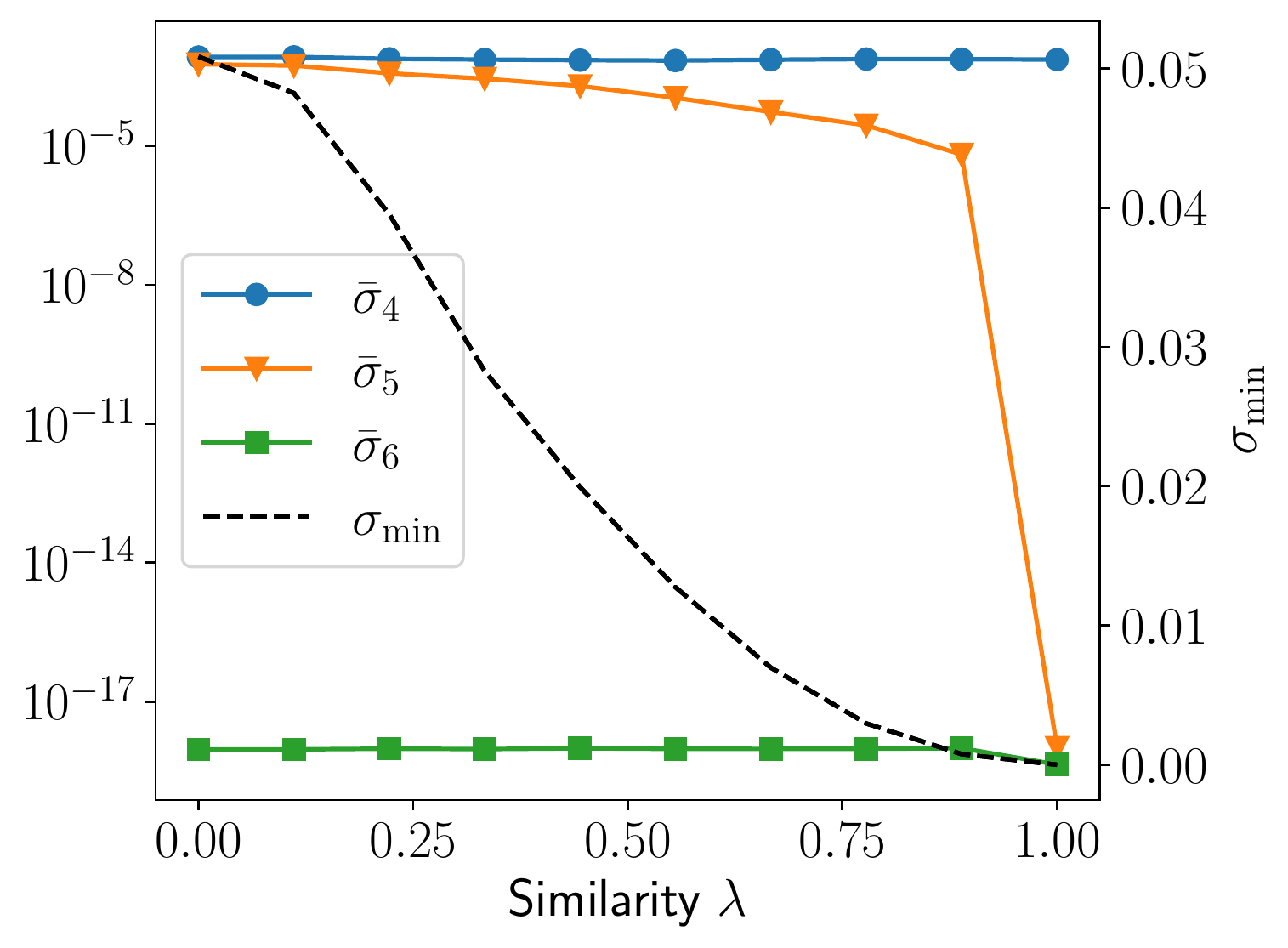} &
        \includegraphics[width=0.23\linewidth]{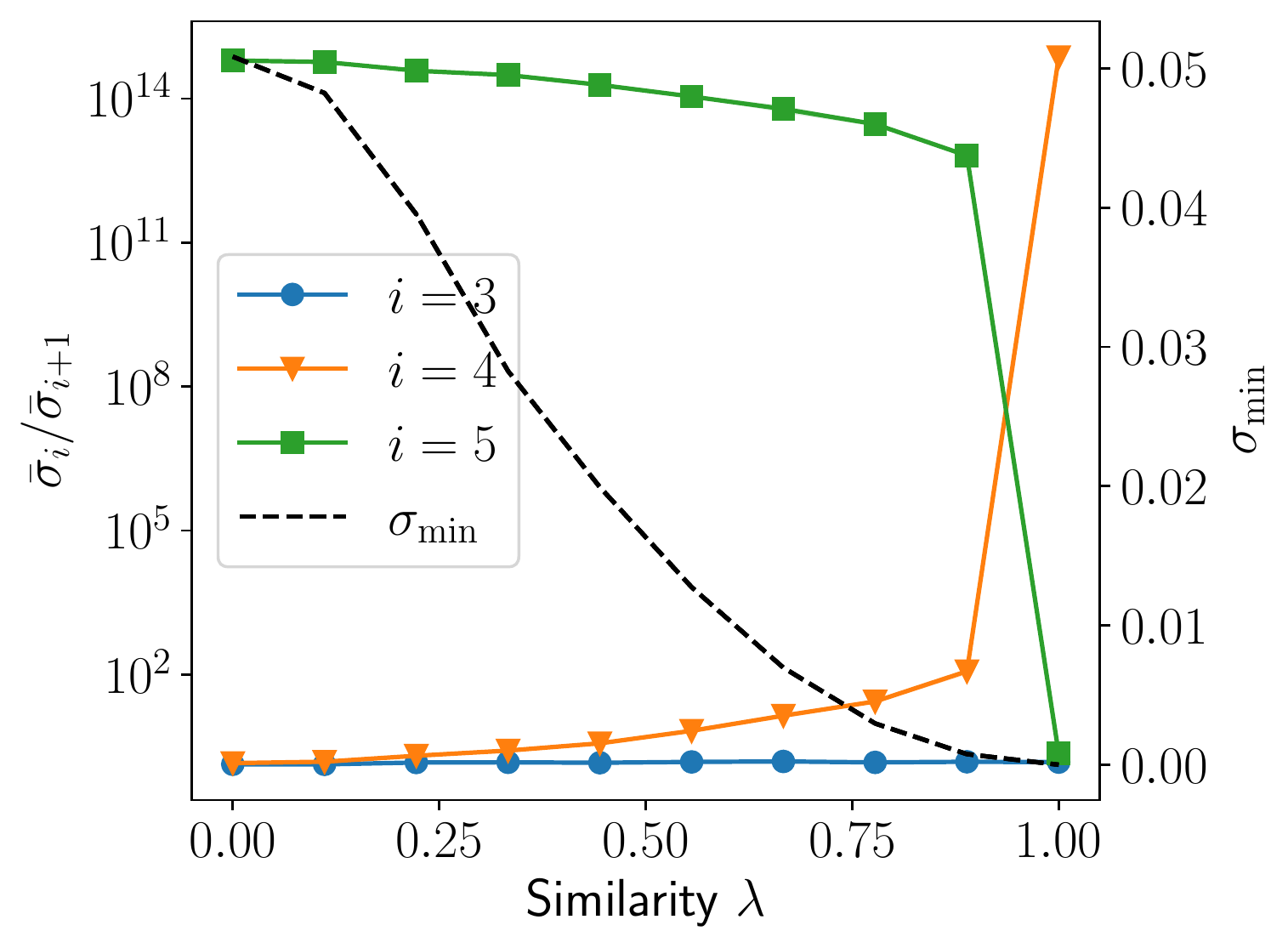} & \includegraphics[width=0.23\linewidth]{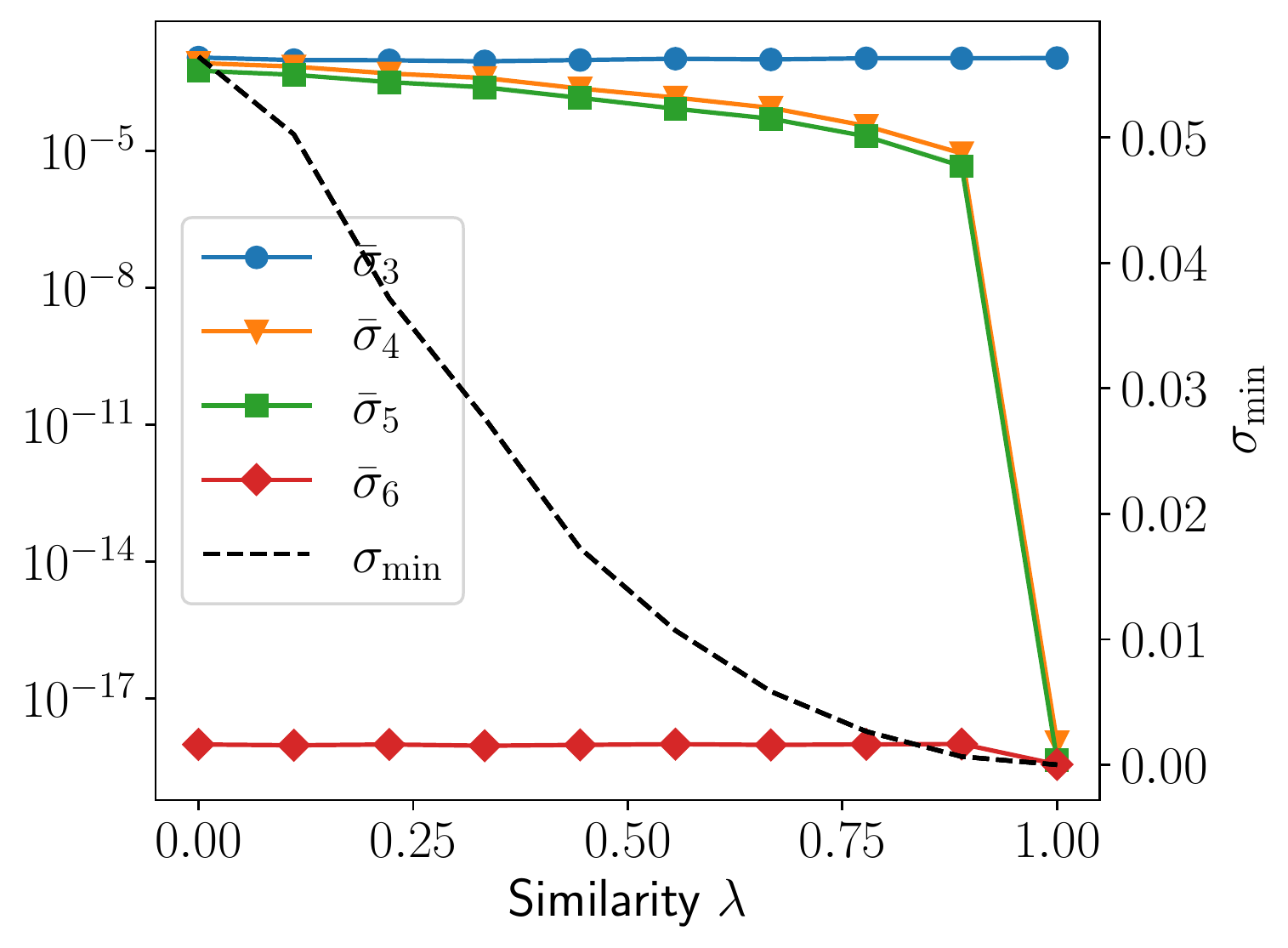} &
        \includegraphics[width=0.23\linewidth]{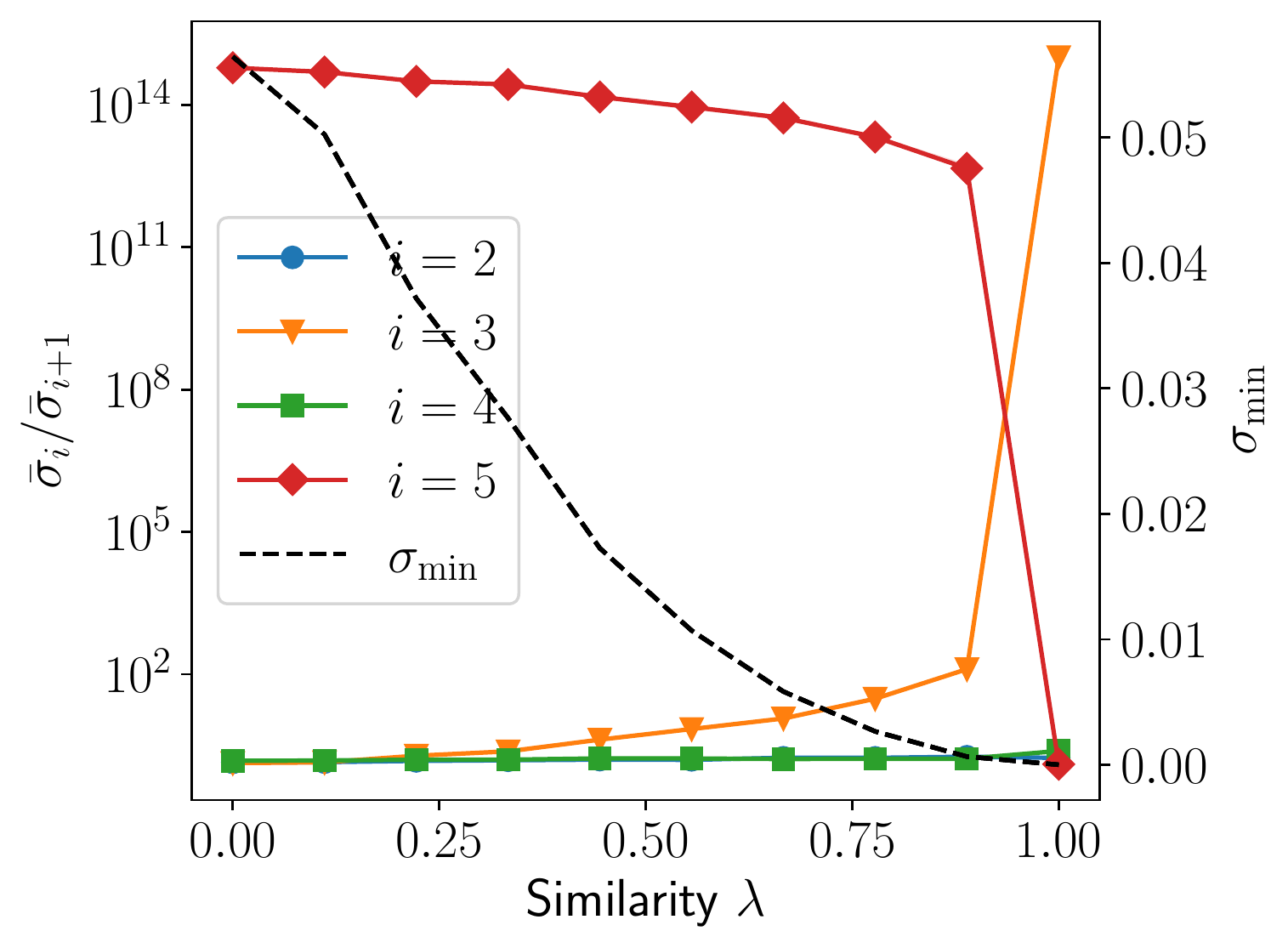} \\
         (a) & (b) &   (c) & (d)  \\ 
    \end{tabular}
    \negfigsp
    \caption{\label{fig:sval} Average singular values for varying similarity $\lambda$,
    on a mixture with $L=5$ chains and $n=20$ states for Scenario 1 (a,b) and Scenario 2 (c,d) respectively.  Plots~(a,c) show the singular values (left axis) and $\sigma_{\min}$ (right axis, see Eq.~\eqref{eq:sigmamin}) and (b,d) their ratios respectively.}
\end{figure}

In addition to the degeneracy described above (Scenario 1), we showcase the effects of additionally making $M^3$ more similar to $M^4$ (Scenario 2). Figures~\ref{fig:sval}(c,d) show our results for the 2nd scenario. Observe the gradual decrease of $\sigma_{\min}$ and the sudden drop of both singular values $\bar \sigma_4$ and $\bar \sigma_5$. The ratio $\bar \sigma_4 / \sigma_5$ only increases slightly, while $\bar \sigma_3 / \bar \sigma_4$ explodes as the degeneracy increases, suggesting that we reduce the number of chains to $L=3$. In general, as a rule of thumb, we suggest that if the $i$-th singular value achieves the maximum ratio $\sigma_i / \sigma_{i+1}$, we should set $L = i$. 
The algorithm's choice of $L$ and recovery  error are shown in Figure~\ref{fig:sval2}.

\begin{figure}[!t]
    \centering
    \begin{tabular}{cc}
        \includegraphics[width=0.45\linewidth]{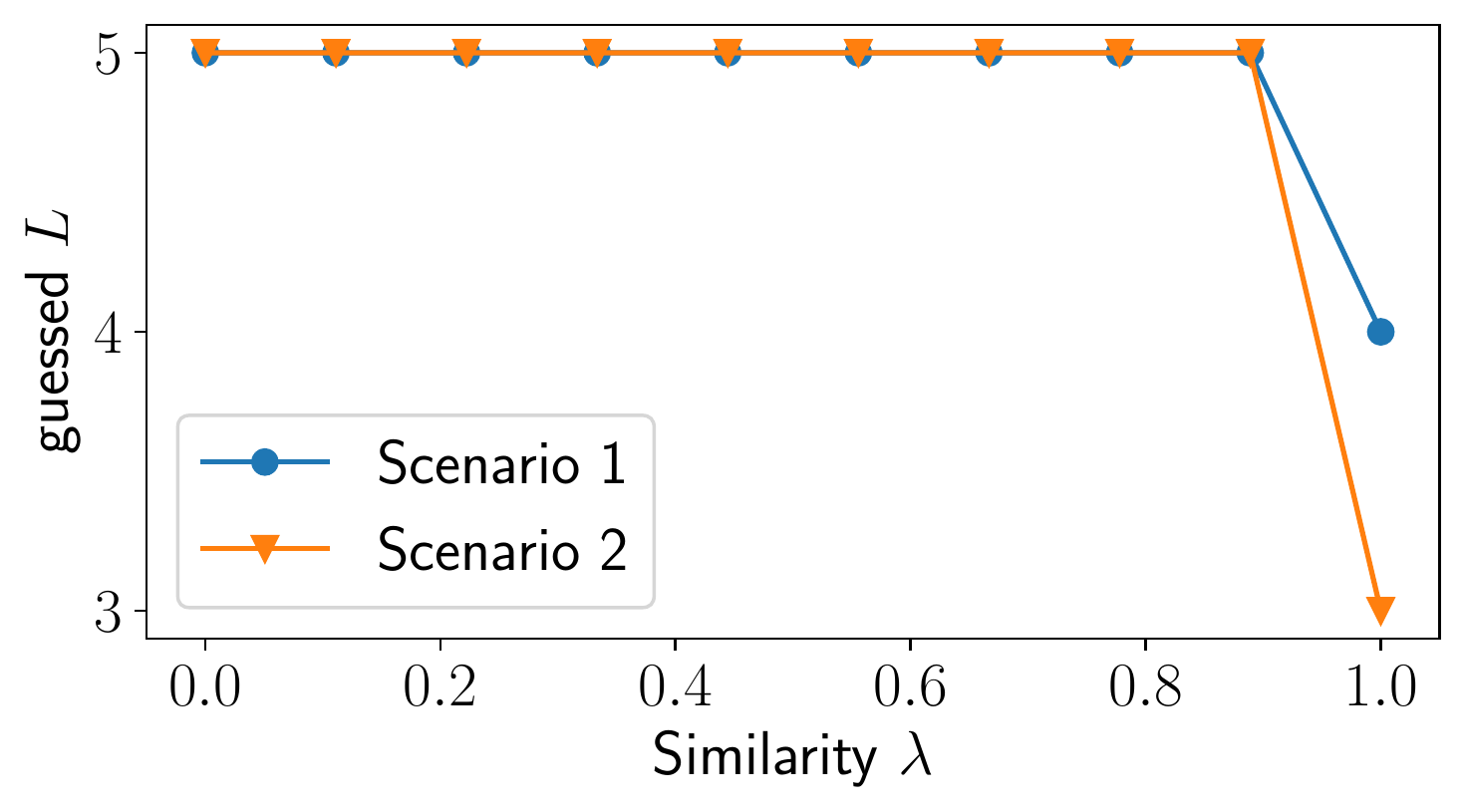} &
        \includegraphics[width=0.45\linewidth]{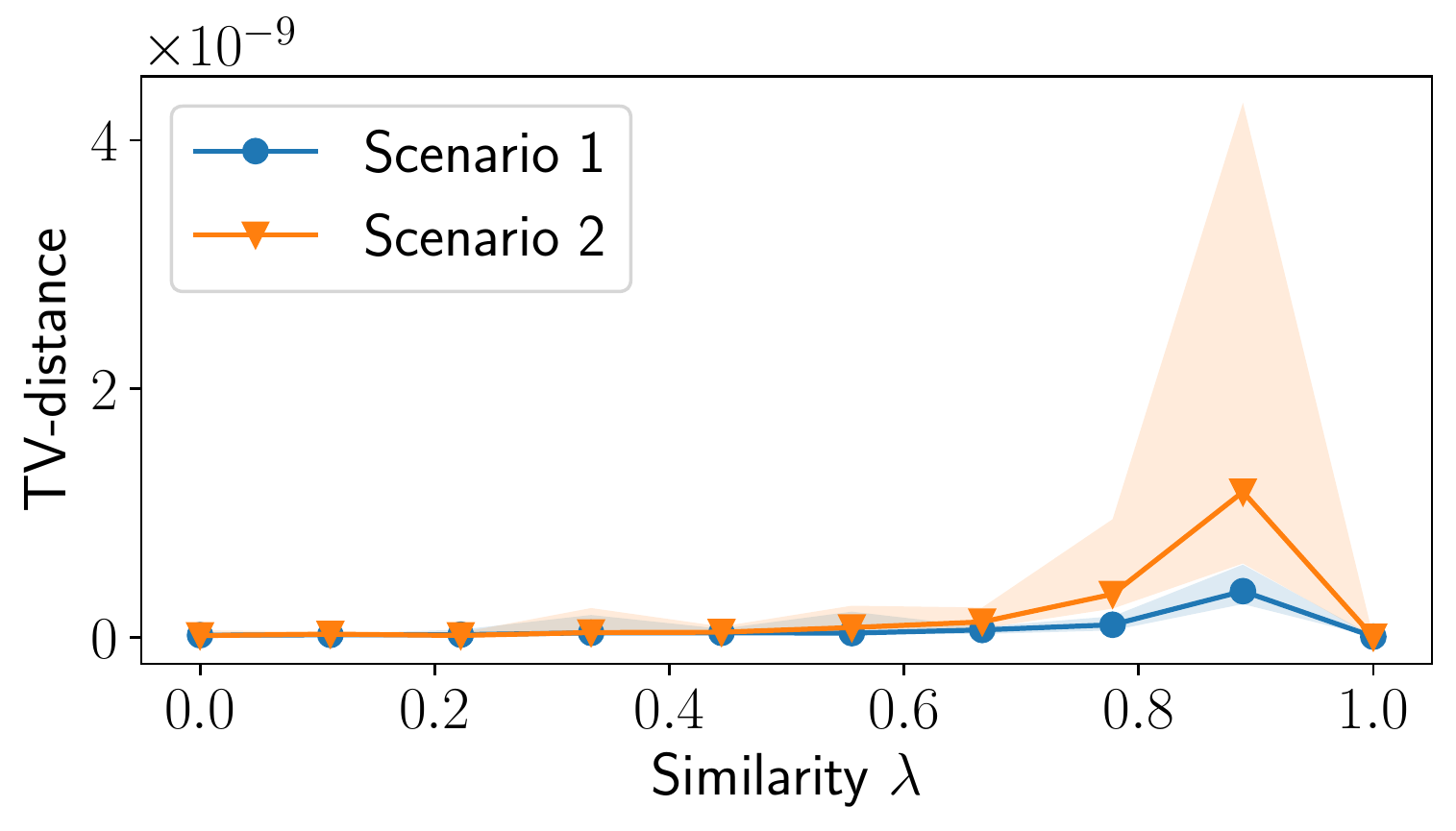} \\
        (a) & (b)
    \end{tabular}
    \negfigsp
    \caption{\label{fig:sval2} (a) Our proposed method for choosing $L$ perfectly guesses the true number of chains in two controlled scenarios $L$ for all values $\lambda$. (b) The TV-distance between the groundtruth and the learned mixture is of the order $10^{-9}$ for all $\lambda$ values. Our reconstruction algorithm $\casvd$ achieves near-perfect recovery.}
\end{figure}

\subsection{Real-world experiments}
\label{subsec:real}


 
\spara{MSNBC} Figure~\ref{fig:real-world}(a) shows the TV-distance of the input and learned 3-trail distributions  for our reconstruction algorithm $\casvd$ combined with $\expm$ for different values of $L$. The legend provides information on the number of iterations we use in the $\expm$ algorithm. We observe that as the number of EM iterations increase to 20 the TV-distance goes down. The SVD-based methods $\casvd$ and $\gkvsvd$ perform  worse than $\expm$ that finds a good local optimum on this dataset; the best solution is to use the output of  $\casvd$   as a starting point for $\expm$. This has the advantage of reducing the number of iterations of EM until convergence, which are expensive.



\begin{figure}[htbp]
    \centering
    \begin{tabular}{cc}
        \includegraphics[width=0.3\linewidth]{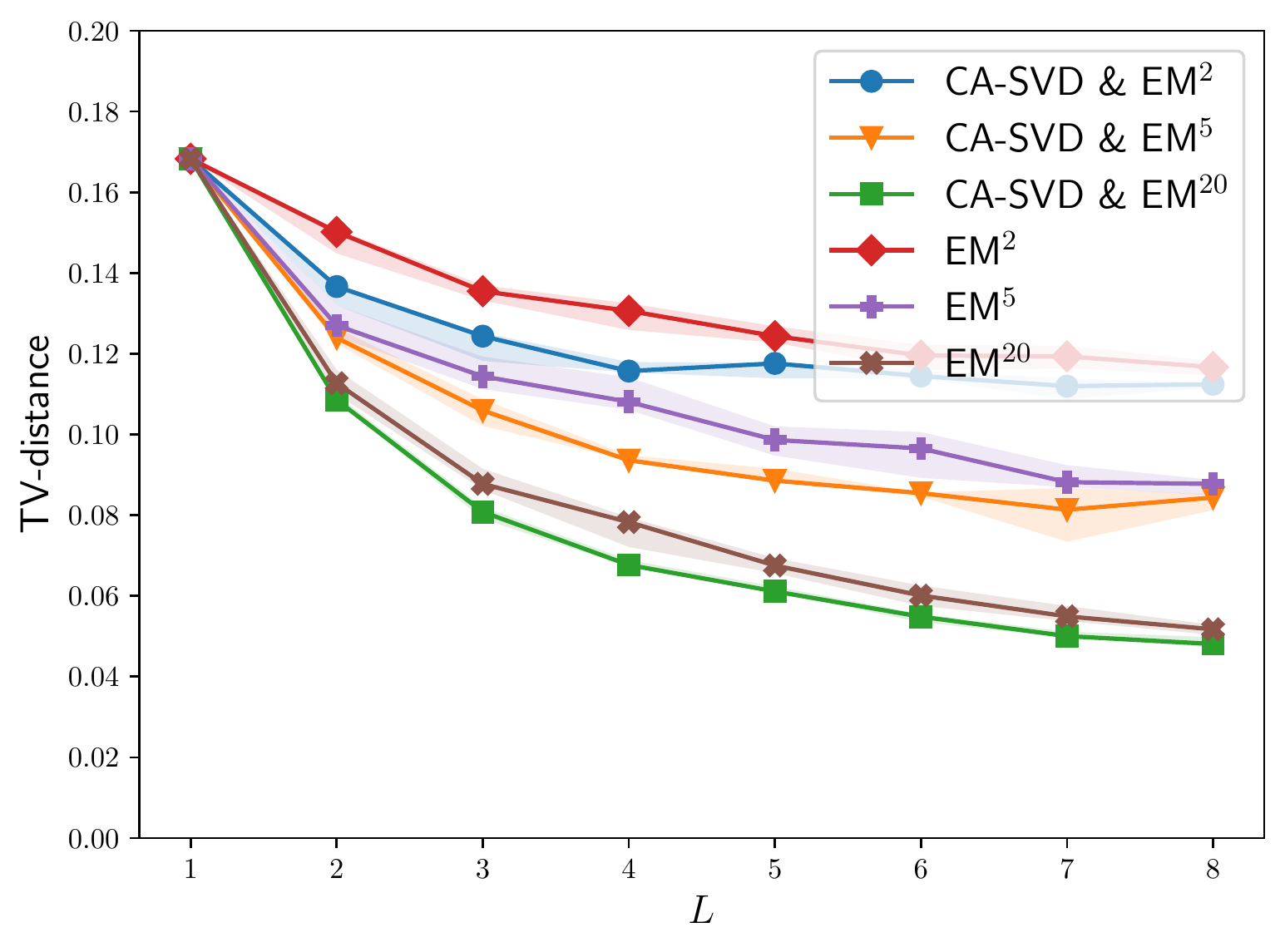} &
        \includegraphics[width=0.31\linewidth]{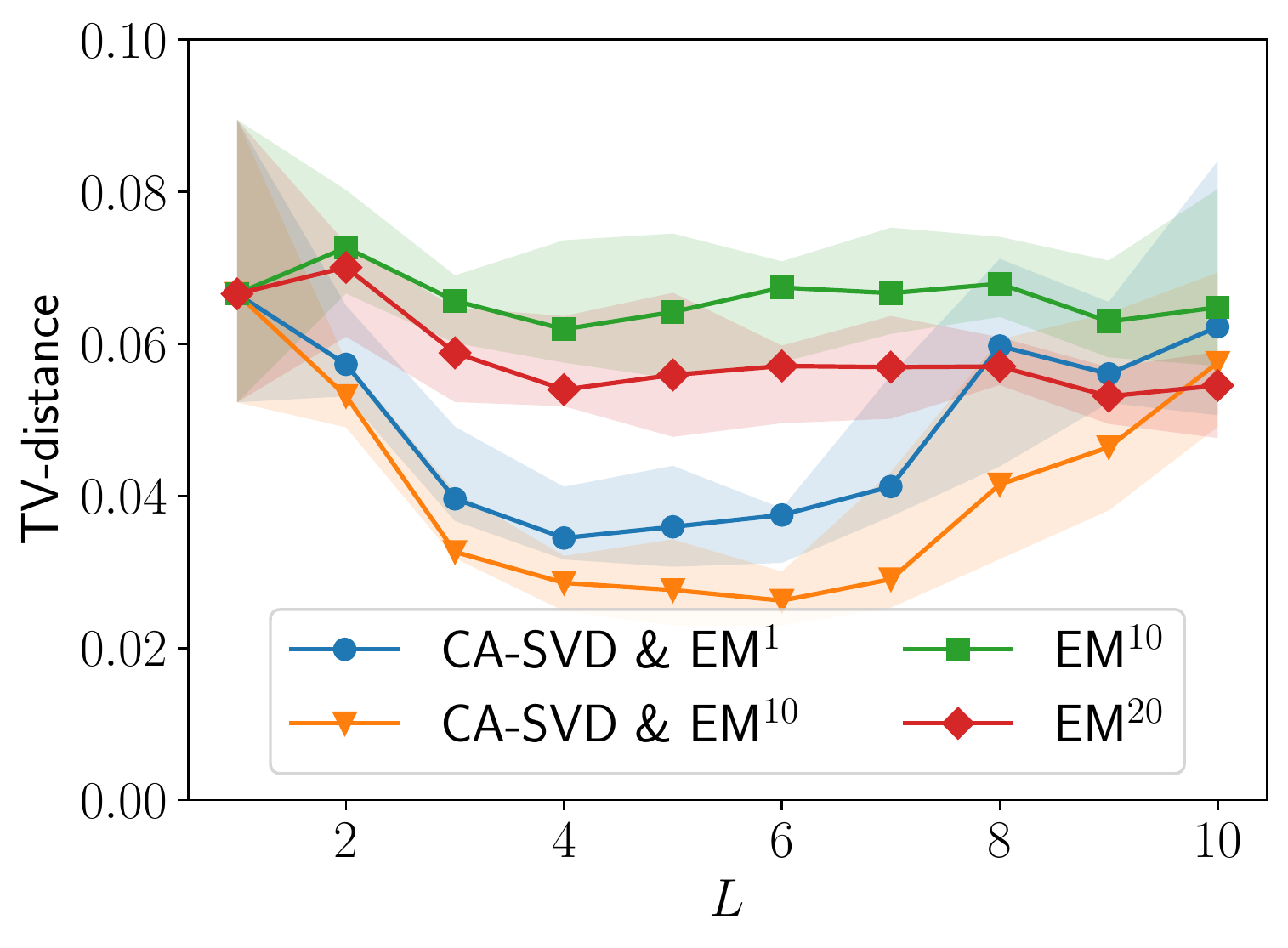} \\
        (a) & (b)
    \end{tabular}
    \negfigsp
    \caption{\label{fig:real-world} Evaluation on the {\it MSNBC} (a) and {\it Mushrooms} dataset (b).
    A combination of our algorithm  $\casvd$ and $\expm$ yields the lowest error.}
\end{figure}



\spara{Mushrooms}  
Figure~\ref{fig:real-world}(b) shows the trail-error of pure $\casvd$, $\expm$,  and $\casvd$ combined with a few iterations of $\expm$ trained on $L \le 10$.
We observe that the $\casvd$ on itself has difficulties for $L > 6$. 
However, the result is still valuable as an initial guess for $\expm$.
In particular, $\casvd$ combined with $\expm$ by far outperforms $\expm$ on a random initialization with just a few iterations.
Note that clustering the chains corresponds to a form of dimensionality reduction,
but since our goal is not to perform state-of-the-art dimensionality reduction, we only use  $\casvd$ and $\expm$.  

\hide{ 
\begin{figure*}[htbp]
    \centering
    \includegraphics[width=0.45\linewidth]{figs/sval2-2022-09-19-13-13-03-573084.pdf}
    \caption{Singular Values for different values of $L$} \label{fig:sval3}
    \fabian{keep?}
\end{figure*}
}

\section{Conclusion} 
\label{sec:concl} 
Learning mixtures of Markov chains from samples is an intriguing problem, both from a pure and applied mathematical perspective. The recent work of Gupta, Kumar and Vassilvitskii~\cite{gupta2016mixtures} provided a state-of-the-art algorithm $\gkvsvd$ that combines rigor and efficiency. Nonetheless, their algorithm comes with certain  restrictions that we discussed in this work. Most importantly, the $\gkvsvd$ cannot recover a mixture when even one chain in the mixture is disconnected. We provided novel definitions, new algorithmic ideas and technical components that enabled us to design the $\casvd$ reconstruction algorithm that can reconstruct mixtures with greater accuracy and robustly to noise. Furthermore, we provided a rule-of-thumb for choosing $L$ the number of chains in the mixture, based on the singular values of certain matrices. We performed an evaluation of our methods on synthetic datasets where we can clearly compare the methods and see their characteristics. Our method is a strict improvement of the $\gkvsvd$ algorithm, both in theory and in practice. In future work we plan to extend our algebraic analysis to prove tighter bounds that can identify $L$ with greater accuracy and explore even less restrictive conditions for exact or partial recovery with guarantees.

\bibliographystyle{alpha}
\bibliography{main}

\newpage
\clearpage

\appendix


\section{Appendix (Example)}
\label{sec:appendix2}
\input{src/appendix2}

\end{document}